\documentclass[final]{colt2026}

\usepackage{amsmath, amssymb, mathtools}
\usepackage{mathrsfs} 
\usepackage{enumitem}
\usepackage{times}

\newtheorem{assumption}[theorem]{Assumption}

\newcommand{\R}{\mathbb{R}}
\newcommand{\E}{\mathbb{E}}
\newcommand{\Pbb}{\mathbb{P}}
\newcommand{\F}{\mathcal{F}}
\newcommand{\StateSpace}{\mathcal{S}}
\newcommand{\Action}{\mathcal{A}}
\newcommand{\Qstar}{Q^*}
\newcommand{\Bellman}{\mathcal{T}}
\newcommand{\OpNet}{\mathcal{N}}

\newcommand{\Tr}{\text{Tr}}
\newcommand{\ind}{\mathbf{1}}

\title[Deep Q-Learning on Holder Spaces]{Deep Q-Learning on H\"older Spaces}

\coltauthor{%
 \Name{Qian Qi} \Email{qiqian@pku.edu.cn}\\
 \addr Peking University%
}

\begin{document}

\maketitle
\begingroup
\renewcommand{\thefootnote}{\fnsymbol{footnote}}
\footnotetext[1]{Accepted for presentation at the Conference on Learning Theory (COLT) 2026.}
\endgroup

\begin{abstract}%
  We study the operator-theoretic core of Q-learning in continuous-time stochastic control with continuous states and actions. In value-based reinforcement learning, each Q-learning or DQN update is built from a Bellman optimality target; our analysis isolates this target in a diffusion setting and studies its regularity and approximation complexity. Under uniform ellipticity and H\"older-regular coefficients, we show that a Bellman update maps bounded inputs into an anisotropic regularity class, smoothing the state variable while leaving only Lipschitz dependence on the action variable. This yields a compact family of Bellman iterates and motivates a tensor-product DeepONet architecture adapted to the mixed regularity of the problem. We then derive explicit approximation and resource bounds, together with a stiffness--complexity trade-off as the time step $\delta \to 0$. The resulting theory makes a direct contribution to Q-learning theory at the level of Bellman target regularity and approximation in continuous stochastic control. At the same time, we do not claim a full convergence theorem for practical sampled Q-learning with exploration, replay, and stochastic gradient updates.
\end{abstract}

\begin{keywords}%
  Q-Learning Theory, Bellman Value Iteration, Neural Operators, Parabolic Smoothing, Stochastic Optimal Control, Approximation Theory.
\end{keywords}

\section{Introduction}

This paper studies the operator-theoretic core of Q-learning in continuous-time stochastic control with continuous states and actions. In value-based reinforcement learning, the central object is the Bellman optimality update for the action-value function; every Q-learning or DQN step is, at heart, an attempt to approximate this update from data. Our focus is on this Bellman target itself in a continuous-time diffusion model, where the coefficients are known and the operator can be analyzed directly. This does not yet cover all ingredients of practical Q-learning, but it does address what we view as one of its core theoretical questions: what regularity and approximation complexity does the Bellman target class possess?

The relationship with Q-learning can be stated concretely. In discrete time, a typical value-based target has the form $r + \beta \max_{a'} Q(x',a')$, namely a stochastic approximation to a Bellman optimality operator. Our continuous-time discretization replaces the single sampled transition by its conditional expectation over a short diffusion step of size $\delta$, thereby isolating the operator-level part of the theory from the separate issues of exploration, replay, and stochastic optimization. This is why we view the paper as a contribution to Q-learning theory, while also being explicit that it operates at the Bellman-operator level.

Classical analyses of value-based methods with function approximation often assume that the transition and reward coefficients are globally Lipschitz in state and action. While convenient, this hypothesis misses a central feature of stochastic control: non-degenerate diffusion can regularize the Bellman update through parabolic smoothing. In PDE language, the Bellman step is tied to a controlled parabolic equation, so rough inputs may become more regular after one application of the operator \citep{krylov1987nonlinear}.

Motivated by this observation, we study the time-discretized Bellman optimality operator $\Bellman_\delta$ under uniform ellipticity and H\"older-regular coefficients. The main message is that the Bellman targets underlying value-based learning in this regime lie in an anisotropic regularity class substantially smaller than a generic Lipschitz class. This creates an approximation-theoretic route toward neural architectures tailored to the regularity of continuous-control Bellman iterates, and hence toward a sharper theory of what Q-learning-style methods are trying to approximate in this setting.

\subsection{Main Contributions}

Our contributions are fourfold. First, we prove that under uniform ellipticity the Bellman operator maps bounded inputs into an anisotropic H\"older class: smooth in the state variable and only Lipschitz in the action variable. This identifies a nontrivial regularity structure for the Bellman targets that arise in continuous-time Q-learning problems.

Second, we show that this regularity statement is already a core theoretical contribution to Q-learning: it describes the target class that any value-function approximator must represent in this diffusion regime. The operator-theoretic viewpoint is therefore not peripheral to Q-learning, but rather isolates its mathematical backbone.

Third, we use this regularity to derive approximation guarantees for a tensor-product DeepONet architecture that approximates Bellman iterates. The architecture is chosen because it matches the anisotropic regularity revealed by the Bellman analysis; the regularity statement itself is architecture-independent, while the theorem provides one constructive approximation route.

Fourth, we quantify the stiffness--complexity trade-off induced by the time step $\delta$. As $\delta \to 0$, the Bellman operator approaches the identity and the smoothing effect weakens, which increases the approximation complexity. This explains why continuous-time limits can become numerically stiff even when the diffusion is regularizing.

Overall, the paper aims to make a precise claim: in continuous stochastic control with uniformly elliptic diffusion, the Bellman targets at the heart of Q-learning enjoy exploitable anisotropic regularity, and this has concrete consequences for approximation theory and architecture design. What we do \emph{not} claim is a full convergence theorem for sampled Q-learning with exploration, replay, target-network heuristics, and stochastic gradient optimization.

\subsection{Literature Review}

Our work sits at the intersection of approximation theory for Deep Learning, stochastic analysis of PDEs, and reinforcement learning. We briefly survey the relevant literature in these domains.

The convergence of Q-learning with function approximation has been a central topic since the early work of \citet{tsitsiklis1997analysis}. The empirical success of DQN was popularized by \citet{mnih2015human}, while rigorous analyses of value-based methods with function approximation often impose strong smoothness or contraction assumptions; see, for example, \citet{fan2020theoretical}. Our work addresses a complementary but, in our view, foundational theoretical question: what regularity and approximation complexity do the Bellman targets themselves possess in a continuous-time stochastic-control regime? Because Q-learning and DQN are built around repeated approximation of Bellman optimality targets, understanding this target class is already a direct contribution to Q-learning theory, even before one adds finite-sample and optimization analysis. The observations of \citet{tallec2019making} about stiffness as $\delta \to 0$, of \citet{qian2025} about regularity collapse near decision boundaries, and of \citet{qi2026viscosity} about viscosity-consistency constraints for deep Q-network discretizations help motivate our complexity discussion. At the same time, to avoid overstatement, we emphasize that our theorem is an operator-level approximation result for Bellman iterates in a diffusion setting, not a full convergence guarantee for practical sampled Q-learning training dynamics.

The correspondence between the Bellman equation and parabolic PDEs is classical, established by \citet{krylov1980controlled} and \citet{fleming2006controlled}. While the existence of viscosity solutions for HJB equations is well-known \citep{crandall1992user}, the application of \emph{interior regularity estimates} to analyze the sample complexity of RL is relatively new. Most existing finite-sample bounds for RL \citep{azar2017minimax, jin2018q} operate in the tabular or linear setting. In contrast, our work explicitly utilizes the anisotropic nature of the Schauder estimates ($C^{2,\alpha}$ in state, $C^{0,1}$ in action) induced by the diffusion term $\Tr(\sigma \sigma^\top D^2 V)$, a concept deeply rooted in the theory of regularization by noise.

Universal Approximation Theorems (UATs) for neural networks have evolved from qualitative density results \citep{cybenko1989approximation, hornik1991approximation} to quantitative complexity bounds. \citet{yarotsky2017error} established the relationship between network depth/width and the Sobolev regularity of the target function, showing that $W \sim \epsilon^{-d/s}$. However, applying these isotropic bounds to $\Qstar$ yields pessimistic estimates due to the joint dimension $d_{\text{state}} + d_{\text{action}}$. Our analysis circumvents this by adopting the anisotropic viewpoint of \citet{mhaskar2016deep} and \citet{poggio2017and}, decoupling the complexity of the smooth state manifold from the non-smooth policy landscape.

The architecture proposed in this paper draws upon the emerging field of Operator Learning. The \emph{Deep Operator Network} (DeepONet), introduced by \citet{lu2021learning}, generalizes the universal approximation theorem to non-linear operators on Banach spaces. Similarly, the \emph{Fourier Neural Operator} (FNO) \citep{li2020fourier} has shown success in solving PDEs by learning the solution operator directly. While these methods are primarily applied to forward problems in fluid dynamics and solid mechanics, we adapt the framework to the backward Bellman flow. Our \emph{Tensor-Product Trunk} modification addresses the specific mixed-regularity requirements of the Control problem, which are not typically present in physical conservation laws.

\section{Preliminaries and Problem Formulation}

\subsection{The Stochastic Control Problem}

Let $(\Omega, \F, \mathbb{F}=\{\F_s\}_{s \ge 0}, \Pbb)$ be a complete filtered probability space satisfying the usual conditions, supporting a standard $d$-dimensional Brownian motion $W = (W_s)_{s \ge 0}$. We consider a finite-horizon Markov Decision Process over the time interval $t \in [0, T]$. The state process $X_s \in \R^d$ is controlled by an action process $a = (a_s)_{s \in [t, T]}$ taking values in a compact metric space $\Action \subset \R^{d_a}$, where $d_a$ denotes the action dimension. We denote the set of admissible control policies by $\mathcal{U}[t, T]$, defined as the space of all $\mathbb{F}$-progressively measurable processes taking values in $\Action$ such that the associated SDE has a unique strong solution.

For a fixed initial time $t \in [0, T)$ and state $x \in \R^d$, the dynamics of the system are governed by the forward It\^o Stochastic Differential Equation (SDE):
\begin{equation} \label{eq:sde}
    \begin{cases}
        dX_s^{t,x,a} = b(s, X_s^{t,x,a}, a_s) \, ds + \sigma(s, X_s^{t,x,a}, a_s) \, dW_s, \quad s \in [t, T], \\
        X_t^{t,x,a} = x.
    \end{cases}
\end{equation}
The objective of the agent is to maximize the expected cumulative reward defined by the cost functional:
\begin{equation} \label{eq:objective}
    J(t, x, a) = \E \left[ \int_t^T r(s, X_s^{t,x,a}, a_s) \, ds + g(X_T^{t,x,a}) \bigg| \F_t \right].
\end{equation}
We assume the drift $b: [0,T] \times \R^d \times \Action \to \R^d$, diffusion $\sigma: [0,T] \times \R^d \times \Action \to \R^{d \times d}$, running reward $r: [0,T] \times \R^d \times \Action \to \R$, and terminal reward $g: \R^d \to \R$ satisfy standard measurability and boundedness conditions (detailed in Assumption \ref{ass:coeff}).

The optimal value function $V^*(t, x)$ is defined via the maximization over admissible controls:
\begin{equation}
    V^*(t, x) = \sup_{a \in \mathcal{U}[t, T]} J(t, x, a).
\end{equation}
It is well-known from classical stochastic control theory \citep{krylov1980controlled} that, under suitable regularity conditions, $V^*$ is the unique viscosity solution to the backward Hamilton-Jacobi-Bellman (HJB) equation:
\begin{equation} \label{eq:HJB_continuous}
    \begin{cases}
        -\partial_t V^*(t,x) - \sup_{a \in \Action} \left\{ \mathcal{L}^a V^*(t,x) + r(t, x, a) \right\} = 0, \quad (t,x) \in [0, T) \times \R^d, \\
        V^*(T, x) = g(x),
    \end{cases}
\end{equation}
where $\mathcal{L}^a$ is the second-order linear differential operator (generator) associated with a constant action $a$:
\begin{equation} \label{eq:generator}
    \mathcal{L}^a \phi(x) \coloneqq \frac{1}{2} \Tr\left( \sigma(t,x,a)\sigma(t,x,a)^\top D^2_x \phi(x) \right) + b(t,x,a)^\top \nabla_x \phi(x).
\end{equation}
In the present idealized Bellman-iteration model, we do not analyze data-driven Q-learning or exploration directly. Instead, we approximate the \emph{action-value function} $Q^*(t,x,a)$ associated with the semi-discrete Bellman operator obtained by holding an action constant for a time step $\delta$ and acting optimally thereafter.

\subsection{Assumptions: Ellipticity and Regularity}

To invoke the Schauder interior regularity estimates, we require specific structural assumptions on the MDP. A critical distinction in our analysis is the separation between regularity in the state-time domain $(t,x)$, where diffusion generates smoothing, and the action space $\Action$, where no such smoothing occurs.

We first define the \emph{parabolic metric} on $[0,T] \times \R^d$:
\begin{equation}
    d_P((t,x), (s,y)) \coloneqq |t-s|^{1/2} + \|x-y\|.
\end{equation}
A function $f: [0,T] \times \R^d \to \R$ is said to be \textbf{parabolic H\"older continuous} with exponent $\alpha \in (0,1)$, denoted $f \in C^{\alpha/2, \alpha}([0,T] \times \R^d)$, if:
\begin{equation}
    \|f\|_{C^{\alpha/2, \alpha}} \coloneqq \sup_{(t,x)} |f(t,x)| + \sup_{(t,x) \ne (s,y)} \frac{|f(t,x) - f(s,y)|}{d_P((t,x), (s,y))^\alpha} < \infty.
\end{equation}
This scaling reflects the natural heat-kernel behavior where time scales as the square of space.

\begin{assumption}[Global Coefficients] \label{ass:coeff}
    The drift $b$, diffusion $\sigma$, and reward $r$ satisfy the following conditions on $[0,T] \times \R^d \times \Action$:
    \begin{enumerate}[label=(\roman*),wide]
        \item \textbf{Global Boundedness and H\"older Continuity:} 
        The functions $b, \sigma, r$ are bounded and uniformly H\"older continuous with exponent $\alpha \in (0,1)$ in the parabolic metric on the entire domain $\mathcal{D}_T = [0,T] \times \R^d$. They are also globally Lipschitz in the state variable, uniformly over time and actions, so that the controlled SDE and the flow-sensitivity estimates used below are well posed. Specifically, there exists a constant $K$ such that for all $a \in \Action$:
        \begin{equation}
            \|b(\cdot, \cdot, a)\|_{C^{\alpha/2, \alpha}(\mathcal{D}_T)} + \|\sigma(\cdot, \cdot, a)\|_{C^{\alpha/2, \alpha}(\mathcal{D}_T)} + \|r(\cdot, \cdot, a)\|_{C^{\alpha/2, \alpha}(\mathcal{D}_T)} \le K.
        \end{equation}
        The terminal payoff $g$ is bounded and globally Lipschitz in $x$.
        \item \textbf{Uniform Ellipticity on $\R^d$:} 
        There exists a constant $\nu > 0$ such that for all $(t,x,a) \in [0,T] \times \R^d \times \Action$ and $\xi \in \R^d$:
        \begin{equation}
            \xi^\top (\sigma \sigma^\top)(t,x,a) \xi \ge \nu \|\xi\|^2.
        \end{equation}
        \item \textbf{Lipschitz Continuity in Actions:} 
        The coefficients are globally Lipschitz continuous with respect to the action parameter $a$, uniformly in $(t,x)$.
    \end{enumerate}
\end{assumption}

\begin{remark}[Domain Definition]
    By assuming the coefficients are well-behaved on $\R^d$, we treat the Bellman update as a Cauchy problem. This avoids the technical complications of boundary conditions (e.g., reflecting or absorbing boundaries) on a compact domain $\Omega$, where interior estimates would diverge as $dist(x, \partial \Omega) \to 0$. The compactness required for the approximation theorem will be recovered via the restriction of the global solution to compact reachable sets.
\end{remark}

\begin{assumption}[Uniform Ellipticity] \label{ass:ellipticity}
    The diffusion is uniformly non-degenerate; this repeats the ellipticity condition separately because it is the assumption invoked directly by the PDE estimates. Let $\Sigma(t,x,a) \coloneqq \sigma(t,x,a)\sigma(t,x,a)^\top$. There exists a constant $\nu > 0$ such that for all $(t,x,a) \in [0,T] \times \R^d \times \Action$ and all $\xi \in \R^d$:
    \begin{equation}
        \xi^\top \Sigma(t,x,a) \xi \ge \nu \|\xi\|^2.
    \end{equation}
\end{assumption}

\begin{remark}[Necessity of Assumptions]
    Assumption \ref{ass:coeff}(i) ensures that the coefficients possess sufficient regularity to support classical solutions to the parabolic PDE. If the coefficients were merely bounded measurable, solutions would only exist in Sobolev spaces $W^{1,2}_p$ \citep{krylov1980controlled}, which is insufficient for the pointwise gradient estimates required for our operator network bounds. Assumption \ref{ass:ellipticity} prevents the diffusions from becoming singular, ensuring the existence of a smooth Gaussian transition density (heat kernel) which drives the regularization mechanism.
\end{remark}

\begin{example}[A controlled diffusion satisfying the assumptions]
A concrete example is a bounded stochastic tracking problem with compact action set $\Action \subset \mathbb{R}^{d_a}$:
\[
    dX_s = \bigl(\tanh(A X_s) + B a_s + \beta(s)\bigr)\,ds + \Sigma_0\,dW_s,
\]
where $A$ and $B$ are fixed matrices, $\beta$ is bounded and smooth, and $\Sigma_0\Sigma_0^\top \ge \nu I_d$ for some $\nu>0$. Let the running reward and terminal payoff be bounded H\"older-continuous functions, for instance
\[
    r(t,x,a) = \bar r\bigl(t,\tanh(x),a\bigr), \qquad g(x)=\bar g(\tanh(x)),
\]
with $\bar r$ and $\bar g$ smooth and bounded and with $\Action$ compact. Then $b$, $\sigma$, and $r$ satisfy Assumption~\ref{ass:coeff}, Assumption~\ref{ass:ellipticity} holds by construction, and the model retains the key features relevant to our setting: continuous states, continuous actions, and genuinely stochastic transitions. The example should be read as a bounded-coefficient proxy for continuous-control problems in which diffusion regularizes the Bellman update.
\end{example}

\subsection{Anisotropic Function Spaces}

Standard isotropic function spaces are ill-suited for the analysis of the Bellman operator in stochastic control. The diffusion term $\sigma \sigma^\top$ provides parabolic smoothing in the state-time domain $(t,x)$, but the dependence on the action $a$ is merely parametric. The optimal policy, and consequently the value function, may exhibit non-smooth behavior with respect to $a$ (e.g., kinks at points where the optimal action switches). To rigorously capture this mixed regularity, we introduce \emph{anisotropic H\"older spaces}.

Let $\mathcal{D}_T = [0,T] \times \R^d$. We denote by $C^{1,2,\alpha}(\mathcal{D}_T)$ the classical parabolic H\"older space consisting of functions $u(t,x)$ such that $\partial_t u$, $\nabla_x u$, and $D^2_x u$ are continuous and bounded, and the highest-order derivatives satisfy the $\alpha$-H\"older condition with respect to the parabolic metric $d_P$. The norm is defined as:
\begin{equation}
    \|u\|_{C^{1,2,\alpha}(\mathcal{D}_T)} \coloneqq \|u\|_\infty + \|\partial_t u\|_{C^{\alpha/2, \alpha}} + \|\nabla_x u\|_{C^{(1+\alpha)/2, 1+\alpha}} + \|D^2_x u\|_{C^{\alpha/2, \alpha}}.
\end{equation}
Here, the H\"older seminorms are defined with respect to $d_P((t,x), (s,y)) = |t-s|^{1/2} + \|x-y\|$. This definition ensures that if $u \in C^{1,2,\alpha}$, it satisfies the requisite regularity for the classical Schauder estimates applied to $\partial_t u + \mathcal{L}^a u = f$.

\begin{definition}[Anisotropic Space $\mathcal{H}^{\alpha, \text{Lip}}$]
    A function $Q: [0,T] \times \R^d \times \Action \to \R$ belongs to the anisotropic space $\mathcal{H}^{\alpha, \text{Lip}}$ if:
    \begin{enumerate}[label=(\roman*),wide]
        \item \textbf{Smoothing Domain:} For every fixed action $a \in \Action$, the map $(t,x) \mapsto Q(t,x,a)$ belongs to the parabolic space $C^{1,2,\alpha}(\mathcal{D}_T)$.
        \item \textbf{Parametric Domain:} For every fixed $(t,x) \in \mathcal{D}_T$, the map $a \mapsto Q(t,x,a)$ is Lipschitz continuous (i.e., belongs to $C^{0,1}(\Action)$).
    \end{enumerate}
    The associated norm is:
    \begin{equation}
        \|Q\|_{\mathcal{H}^{\alpha, \text{Lip}}} \coloneqq \sup_{a \in \Action} \|Q(\cdot, \cdot, a)\|_{C^{1,2,\alpha}(\mathcal{D}_T)} + \sup_{(t,x) \in \mathcal{D}_T} \sup_{a \ne a'} \frac{|Q(t,x,a) - Q(t,x,a')|}{\|a-a'\|_\Action}.
    \end{equation}
\end{definition}

This construction is critical for our main result. While the Bellman operator lifts the regularity of the value function in state space (from $L^\infty$ to $C^{2,\alpha}$), it preserves (but does not improve) the regularity with respect to actions. Since the coefficients $b, \sigma, r$ are assumed to be Lipschitz in $a$ (Assumption \ref{ass:coeff}), the solution $Q$ naturally resides in $\mathcal{H}^{\alpha, \text{Lip}}$.

\subsection{The Bellman Semi-Group}

The central object of our analysis is the evolution of the action-value function under the optimal control dynamics. Unlike the infinite-horizon setting which relies on contraction properties of a static operator, the finite-horizon setting is naturally described by a backward parabolic evolution.

Let $\mathcal{L}^a_t$ denote the infinitesimal generator associated with a fixed action $a \in \Action$ at time $t$, acting on functions $\phi \in C^2(\R^d)$:
\begin{equation} \label{eq:generator_def}
    \mathcal{L}^a_t \phi(x) \coloneqq \frac{1}{2} \Tr\left( \sigma(t,x,a)\sigma(t,x,a)^\top D^2_x \phi(x) \right) + b(t,x,a)^\top \nabla_x \phi(x).
\end{equation}

We define the discrete-time backward Bellman operator $\Bellman_\delta$ which maps a terminal value function approximation at time $t+\delta$ to a value function at time $t$. Let $\Phi: \R^d \times \Action \to \R$ be a bounded measurable function representing the Q-function at the future time step $t+\delta$. The update consists of two stages:

\begin{enumerate}[label=(\roman*),wide]
    \item \textbf{Non-linear Maximization (Optimization Step):} We first compute the optimal value at the future state:
    \begin{equation}
        V_{\Phi}(y) \coloneqq \sup_{a' \in \Action} \Phi(y, a').
    \end{equation}
    Note that even if $\Phi$ is smooth, $V_{\Phi}$ is typically only Lipschitz continuous due to the supremum operation.
    
    \item \textbf{Linear Evolution (Smoothing Step):} We solve the backward Kolmogorov equation (Feynman-Kac formula) over the interval $[t, t+\delta)$. For a fixed current action $a$, let $u(\cdot, \cdot; a)$ be the solution to the linear parabolic PDE:
    \begin{equation} \label{eq:backward_pde}
        \begin{cases}
            \partial_s u(s, x) + \mathcal{L}^a_s u(s, x) + r(s, x, a) = 0, & s \in [t, t+\delta), \\
            u(t+\delta, x) = V_{\Phi}(x).
        \end{cases}
    \end{equation}
\end{enumerate}

The Bellman Operator $\Bellman_{t, \delta}$ is defined as the evaluation of this solution at the initial time $s=t$:
\begin{equation} \label{eq:bellman_def}
    (\Bellman_{t, \delta} \Phi)(x, a) \coloneqq u(t, x; a).
\end{equation}

\begin{remark}[Probabilistic Interpretation]
    Under the uniform ellipticity condition (Assumption \ref{ass:ellipticity}), there exists a smooth transition density function (fundamental solution) $\Gamma^a(s, x; \tau, y)$ associated with the operator $\partial_s + \mathcal{L}^a_s$. The operator admits the integral representation:
    \begin{equation}
        (\Bellman_{t, \delta} \Phi)(x, a) = \int_{\R^d} \Gamma^a(t,x; t+\delta, y) V_{\Phi}(y) \, dy + \int_t^{t+\delta} \int_{\R^d} \Gamma^a(t,x; s, y) r(s,y,a) \, dy \, ds.
    \end{equation}
    The first term represents the diffusion of the terminal value (convolution with a heat kernel), while the second term accumulates the running reward. It is the smoothing property of the kernel $\Gamma^a$, specifically its Gaussian decay and singularity behavior, that allows us to recover $C^{2,\alpha}$ regularity for $u(t, \cdot)$ even when the terminal data $V_{\Phi}$ is non-differentiable.
\end{remark}

\section{Regularity of the Bellman Operator}\label{sec:Regularity}

In this section, we establish the smoothing properties of the Bellman operator. Unlike standard RL theory which relies on the contraction property to preserve Lipschitz continuity, we utilize the regularizing effect of the diffusion term $\sigma \sigma^\top$.

We analyze the single-step backward operator $\Bellman_\delta$ defined in \eqref{eq:bellman_def}. This operator can be decomposed into two distinct stages:
\begin{enumerate}[label=(\roman*),wide]
    \item \textbf{Non-linear Maximization:} $V(x) = \sup_{a' \in \Action} \Phi(x, a')$. This step generally \emph{destroys} differentiability, even if $\Phi$ is smooth.
    \item \textbf{Linear Smoothing (Convolution):} The evolution of $V(x)$ backwards in time via the parabolic PDE. This step \emph{restores} differentiability with respect to $x$.
\end{enumerate}

\subsection{Interior Schauder Estimates}\label{sec:Interior}

We analyze the regularity of the single-step Bellman update by casting it as a solution to a linear parabolic Cauchy problem. A critical improvement over standard analysis arises from the observation that while the maximization step destroys differentiability (preventing $C^1$ propagation), it preserves Lipschitz continuity. We leverage this to derive tighter smoothing estimates than those available for generic bounded measurable initial data.

Let $a \in \Action$ be a fixed action. Let $\Phi$ be a bounded Q-function approximation at the next time step, uniformly Lipschitz in the state variable as in Theorem \ref{thm:smoothing}. The update involves first computing the optimal value at the future state:
\begin{equation}
    V_{\Phi}(x) \coloneqq \sup_{a' \in \Action} \Phi(x, a').
\end{equation}
We emphasize that due to the supremum operator, $V_{\Phi}$ typically lacks differentiability (it is generally only Lipschitz continuous even if $\Phi$ is smooth). This roughness prevents the existence of classical solutions at $\tau=0$. However, the diffusion term induces immediate smoothing for $\tau > 0$.

Let $v(\tau, x) \coloneqq (\Bellman_{t+\delta-\tau, \tau} \Phi)(x, a)$. Then $v$ satisfies the forward parabolic PDE with rough initial data:
\begin{equation} \label{eq:pde_forward}
    \begin{cases}
    \partial_\tau v(\tau, x) = \mathcal{L}^a v(\tau, x) + r(t+\delta-\tau, x, a) & \text{in } (0, \delta] \times \R^d, \\
    v(0, x) = V_{\Phi}(x).
    \end{cases}
\end{equation}
Here, the operator $\mathcal{L}^a$ acts on the spatial variable $x$. Under Assumption \ref{ass:ellipticity} (Uniform Ellipticity) and Assumption \ref{ass:coeff} (H\"older continuous coefficients), the operator $\partial_\tau - \mathcal{L}^a$ is strictly parabolic. We invoke interior Schauder estimates to bound the solution at time $\tau = \delta$.

\begin{theorem}[Refined Interior Parabolic Smoothing] \label{thm:smoothing}
    Grant Assumptions \ref{ass:coeff} and \ref{ass:ellipticity}. Let the input function $\Phi$ be uniformly Lipschitz continuous in the state variable, i.e., $\sup_{a' \in \Action} \|\Phi(\cdot, a')\|_{C^{0,1}(\R^d)} \le L_{\Phi}$.
    
    Then, for any fixed $\delta > 0$, the function $x \mapsto (\Bellman_\delta \Phi)(t, x, a)$ belongs to the H\"older space $C^{2,\alpha}(\R^d)$. Moreover, there exists a constant $C > 0$ (depending only on $d$, $\nu$, and coefficient bounds $K$) such that the following tighter smoothing estimate holds:
    \begin{equation} \label{eq:schauder_bound}
        \| (\Bellman_\delta \Phi)(t, \cdot, a) \|_{C^{2,\alpha}(\R^d)} \le C \left( \|r\|_{C^{\alpha/2,\alpha}} + \delta^{-\frac{1+\alpha}{2}} \sup_{a' \in \Action} \|\Phi(\cdot, a')\|_{C^{0,1}(\R^d)} \right).
    \end{equation}
    Critically, the singularity as $\delta \to 0$ scales as $\mathcal{O}(\delta^{-\frac{1+\alpha}{2}})$, which is strictly milder than the $\mathcal{O}(\delta^{-\frac{2+\alpha}{2}})$ scaling associated with discontinuous initial data.
\end{theorem}

\begin{proof}
    See Appendix \ref{app:smoothing}.
\end{proof}

\begin{remark}[Improved Stiffness Scaling]
    The refined estimate in Theorem \ref{thm:smoothing} has significant implications for the complexity analysis. The radius of the reachable function set scales as $R_\delta \sim \delta^{-\frac{1+\alpha}{2}}$. Accordingly, throughout the remainder of the paper we use the stiffness scaling factor
    \begin{equation}
        \gamma(\alpha) = \frac{3+\alpha}{2(2+\alpha)}.
    \end{equation}
    This is the quantity appearing in Theorem \ref{thm:convergence} and in the entropy bounds of the appendix. For $\alpha \approx 1$, it reduces the exponent of the complexity explosion by approximately $17\%$, reflecting the benefit of propagating Lipschitz continuity through the max-operator.
\end{remark}

\begin{remark}[Regularity Bootstrap]
    Crucially, Theorem \ref{thm:smoothing} guarantees that even if the input $\Phi$ is initialized as a rough approximation (e.g., from a neural network with random weights), the output after a single Bellman step lies in $C^{2,\alpha}(\R^d)$. This healing property ensures that the sequence of value iterates remains within a compact subset of the function space (for fixed $\delta$), satisfying the prerequisites for the Operator Approximation Theorem.
\end{remark}

\subsection{Domain Truncation and Concentration of Measure}

A central technical challenge in applying global PDE estimates to Deep Reinforcement Learning is the bounded domain of the training data. While the theoretical diffusion process defined in \eqref{eq:sde} is supported on the entire space $\R^d$, practical neural networks are trained on compact sets. Applying the Bellman operator restricted to a compact domain introduces a boundary pollution error, as the true value $Q(x,a)$ depends on the probability mass that diffuses outside the domain.

To resolve the inconsistency between the global Cauchy problem (analyzed in Section \ref{sec:Interior}) and the compact function support required for complexity analysis, we invoke the concentration of measure phenomenon associated with uniform ellipticity. We define a computational domain strictly larger than the target inference domain and prove that the error induced by truncating the integral outside this inflated domain decays exponentially with the margin width relative to the diffusion horizon $\sqrt{\delta}$.

\begin{definition}[Nested Domains]
    Let $\mathcal{S} \subset \R^d$ be the compact \emph{Target Domain} where we require an accurate approximation of $\Qstar$ (e.g., the region of interest for the policy). For a margin parameter $\rho > 0$, we define the \emph{Computational Domain} $\mathcal{S}_\rho$ as the closed $\rho$-neighborhood of $\mathcal{S}$:
    \begin{equation}
        \mathcal{S}_\rho \coloneqq \left\{ y \in \R^d : \inf_{x \in \mathcal{S}} \|y - x\|_2 \le \rho \right\}.
    \end{equation}
    We assume the sensor set $Z_m$ used for training covers the inflated domain $\mathcal{S}_\rho$, while the error analysis is restricted to the interior $\mathcal{S}$.
\end{definition}

We define the \emph{Truncated Bellman Operator} $\Bellman_\delta^\rho$, which integrates the target value only over trajectories that remain within the computational domain. Let $\tau_\rho \coloneqq \inf \{ s > t : X_s \notin \mathcal{S}_\rho \}$ be the first exit time. The truncated operator is defined implicitly by evaluating the conditional expectation on the event $\{ \tau_\rho > t+\delta \}$.

The following lemma quantifies the probability of escaping the computational domain within one time-step $\delta$.

\begin{lemma}[Sub-Gaussian Exit Probability] \label{lemma:concentration}
    Under Assumption \ref{ass:ellipticity} (Uniform Ellipticity) and the boundedness of the drift $b$ and diffusion $\sigma$, the process $X_s$ exhibits sub-Gaussian concentration. Let $M_b \coloneqq \sup_{(t,x,a)} \|b(t,x,a)\|_2$ and $\Lambda \coloneqq \sup_{(t,x,a)} \|\sigma(t,x,a)\|_F$.
    
    For any $x \in \mathcal{S}$, fixed $\delta > 0$, and margin $\rho > M_b \delta$, the probability that the diffusion exits $\mathcal{S}_\rho$ satisfies:
    \begin{equation}
        \Pbb\left( \tau_\rho \le t+\delta \mid X_t = x \right) \le 2d \cdot \exp\left( - \frac{(\rho - M_b \delta)^2}{2 d \Lambda^2 \delta} \right).
    \end{equation}
\end{lemma}

\begin{proof}
    Consider the integral form of the dynamics:
    \[ X_{t+\delta} - x = \int_t^{t+\delta} b(s, X_s, a) \, ds + \int_t^{t+\delta} \sigma(s, X_s, a) \, dW_s. \]
    By the triangle inequality, the exit event implies a large deviation of the running martingale supremum:
    \[ \{\tau_\rho \le t+\delta\} \subseteq \left\{ \sup_{0 \le s \le \delta} \left\| \int_t^{t+s} \sigma(u,X_u,a) \, dW_u \right\| \ge \rho - M_b \delta \right\}. \]
    Let $M_s = \int_t^{t+s} \sigma(u,X_u,a) \, dW_u$. This is a continuous vector-valued local martingale with coordinate quadratic variations bounded by $\Lambda^2 \delta$. If $\sup_s \|M_s\|_2 \ge \lambda$, then some coordinate satisfies $\sup_s |M_s^{(k)}| \ge \lambda/\sqrt{d}$. Applying the maximal Bernstein inequality for continuous martingales to each coordinate (see, e.g., \cite{revuz2013continuous}, Ch. IV) and taking a union bound gives
    \[ \Pbb\left(\sup_{0 \le s \le \delta} \|M_s\|_2 \ge \lambda\right) \le 2d \exp\left(-\frac{\lambda^2}{2d\Lambda^2\delta}\right). \]
    Setting $\lambda=\rho-M_b\delta$ yields the claimed bound.
\end{proof}

\begin{theorem}[Truncation Error and Interior Validity] \label{thm:truncation}
    Let $Q \in L^\infty(\R^d \times \Action)$ be bounded, and set $M_{val} \coloneqq \|V_\Phi\|_\infty + \delta\|r\|_\infty$ for the one-step payoff bound. Let $\Bellman_\delta$ be the true global operator and $\Bellman_\delta^\rho$ be the operator truncated to $\mathcal{S}_\rho$.
    To ensure the boundary pollution error satisfies $\| \Bellman_\delta Q - \Bellman_\delta^\rho Q \|_{L^\infty(\mathcal{S})} \le \varepsilon_{trunc}$, it suffices to choose the computational margin $\rho$ scaling as:
    \begin{equation} \label{eq:margin_scaling}
        \rho(\delta, \varepsilon_{trunc}) \ge M_b \delta + \Lambda \sqrt{2d \delta \log\left( \frac{2 d M_{val}}{\varepsilon_{trunc}} \right)}.
    \end{equation}
    Critically, with this margin, the \emph{Interior} Schauder estimates on $\mathcal{S}$ are preserved. The global smoothing estimates from Theorem \ref{thm:smoothing} apply to $\Bellman_\delta Q|_\mathcal{S}$ with negligible perturbation from the boundary $\partial \mathcal{S}_\rho$.
\end{theorem}

\begin{proof}
    See Appendix \ref{app:truncation}.
\end{proof}

\begin{remark}[Validity of Interior Estimates]
    Standard interior Schauder estimates on a domain $\mathcal{S}$ depend on the distance to the boundary of the domain where the PDE is solved, scaling roughly as $\text{dist}(\mathcal{S}, \partial \mathcal{S}_\rho)^{-k}$. By choosing $\rho \sim \sqrt{\delta}$, we ensure that the boundary of the computational domain is located several standard deviations away from the target domain. Consequently, the transition density $\Gamma(t,x; t+\delta, y)$ is effectively zero for $x \in \mathcal{S}$ and $y \in \partial \mathcal{S}_\rho$. This justifies using the global regularity constant $K$ (from Theorem \ref{thm:smoothing}) for functions restricted to $\mathcal{S}$, as the boundary layer effects are suppressed by the Gaussian tail decay.
\end{remark}

\subsection{Compactness of Q-Iterates}

We now consider the sequence of Bellman iterates defined by backward induction on a time grid $t_k = T - k\delta$ for $k=0, \dots, N$. Let $Q^{(0)}(x,a) = g(x)$. The approximate value functions are generated by:
\begin{equation}
    Q^{(k+1)} = \Bellman_\delta Q^{(k)}.
\end{equation}

A naive analysis might suggest that repeated applications of the Bellman operator degrade regularity due to the non-smooth maximization step $\sup_{a'} Q^{(k)}(x, a')$. However, the parabolic smoothing established in Theorem \ref{thm:smoothing} ensures that regularity is not only preserved but \emph{reset} at each step.

\begin{corollary}[Anisotropic Compactness of Iterates] \label{cor:compactness}
    Let $\StateSpace \subset \R^d$ be the compact state domain. Under Assumptions \ref{ass:coeff} and \ref{ass:ellipticity}, for a fixed time-step $\delta > 0$ and finite horizon $T=N\delta$:
    \begin{enumerate}[label=(\roman*),wide]
        \item \textbf{Uniform Spatial Regularity:} The sequence of iterates $\{Q^{(k)}\}_{k=0}^{N}$ (where $Q^{(k)}$ approximates $Q^*(t_k, \cdot, \cdot)$) is uniformly bounded in the anisotropic H\"older space on the state domain. Specifically, there exists $C_\delta$ such that for all $k$:
        \begin{equation}
            \sup_{a \in \Action} \|Q^{(k)}(\cdot, a)\|_{C^{2,\alpha}(\StateSpace)} + \sup_{x \in \StateSpace} \|Q^{(k)}(x, \cdot)\|_{C^{0,1}(\Action)} \le C_\delta.
        \end{equation}
        
        \item \textbf{Sequential Compactness:} The finite family $\mathcal{Q}^*_\delta = \{Q^{(k)}(\cdot, \cdot)\}_{0 \le k \le N}$ lies in a pre-compact subset of the Banach space $C^{1}(\StateSpace) \times C^0(\Action)$. Equivalently, any sequence drawn from the corresponding uniformly bounded family admits a limit point in the topology of uniform convergence of values and state-gradients.
    \end{enumerate}
\end{corollary}

\begin{proof}
    See Appendix \ref{app:compactness}.
\end{proof}

\subsection{Implication for Neural Approximation}

The compactness result (Corollary \ref{cor:compactness}) provides the rigorous foundation for approximating $\Qstar$ using deep neural networks. Standard universal approximation theorems often assume the target function lies in a generic Sobolev space $W^{k,p}$, leading to pessimistic bounds on the required network size due to the Curse of Dimensionality. Our PDE analysis offers a more precise characterization via the concept of \emph{Metric Entropy}.

\subsubsection*{Metric Entropy and Covering Numbers}
Let $\mathcal{Q}^*_\delta \subset \mathcal{H}^{\alpha, \text{Lip}}$ be the set of reachable Q-functions for a fixed time-step $\delta$. The complexity of learning this class is governed by its Kolmogorov $\epsilon$-entropy, $\log \mathcal{N}(\epsilon, \mathcal{Q}^*_\delta)$, where $\mathcal{N}$ is the covering number (the minimal number of balls of radius $\epsilon$ required to cover the set).

For the H\"older ball of radius $R$ in dimension $d$ with smoothness $s$, the entropy scales as:
\begin{equation}
    \log \mathcal{N}(\epsilon, \text{Ball}_R(C^s)) \asymp \left( \frac{R}{\epsilon} \right)^{d/s}.
\end{equation}
In our context, the effective smoothness in the state space is $s = 2+\alpha$. This implies that the diffusion term effectively reduces the complexity of the function class compared to the standard Lipschitz assumption ($s=1$), where the exponent would be $d$.

\subsubsection*{The Complexity-Stability Trade-off}
However, we must address the dependency on the time discretization $\delta$. From the smoothing estimate \eqref{eq:schauder_bound}, the radius of the function class scales as $R \approx O(\delta^{-\frac{1+\alpha}{2}})$. Consequently, the network complexity required to achieve an error $\epsilon$ scales roughly as:
\begin{equation}
    \text{Complexity}(\delta, \epsilon) \sim \left( \frac{1}{\delta^{\frac{1+\alpha}{2}} \epsilon} \right)^{\frac{d}{2+\alpha}}.
\end{equation}
This reveals a fundamental trade-off:
\begin{itemize}[wide]
    \item \textbf{Smoothing Benefit:} For a fixed $\delta$, the parabolic nature of the Bellman operator constrains the target functions to a smooth manifold, making them efficiently learnable.
    \item \textbf{Stiffness Cost:} As $\delta \to 0$ (approximating the continuous-time control limit), the Bellman operator approaches the Identity, and the smoothing effect acts over an infinitesimally small time. The functions become stiff (large second derivatives), inflating the constant $R$ and requiring larger networks (increased width/depth) to resolve the sharp curvatures near decision boundaries.
\end{itemize}

\subsubsection*{Anisotropic Inductive Bias}
The distinction between state and action regularity dictates a specific inductive bias for the network architecture. A standard Multi-Layer Perceptron (MLP) treating $(x,a)$ as a concatenated vector represents an isotropic prior.
Corollary \ref{cor:compactness} suggests that the optimal architecture should be \textbf{anisotropic}:
\begin{itemize}[wide]
    \item The state subnet should utilize smooth activation functions (e.g., $\tanh$, SiLU) to match the $C^{2,\alpha}_x$ regularity.
    \item The action subnet should utilize piecewise-linear activations (e.g., ReLU) to model the potentially non-smooth ($C^{0,1}_a$) dependence and the sharp transitions induced by the $\max$ operator.
\end{itemize}
This motivates the specific design of the Operator-Network proposed in Appendix \ref{sec:Architecture}.

\section{Universal Approximation Analysis}

In this section, we establish the convergence rate of the proposed Neural Operator scheme. A rigorous analysis must address the \emph{anisotropic} nature of the value function regularity established in Section \ref{sec:Regularity}: the Bellman iterates exhibit parabolic smoothing in the state space $\StateSpace$ but remain merely Lipschitz continuous in the action space $\Action$. Standard isotropic approximations (e.g., standard MLPs) suffer from the curse of dimensionality over the combined dimension $d + d_a$. We prove that our Tensor-Product DeepONet architecture (Section \ref{sec:Anisotropic}) successfully \emph{decouples} these complexities, allowing the parabolic smoothing to reduce the effective dimension of the state space without being bottlenecked by the action space geometry.

\subsection{Main Theorem: Decoupled Complexity Analysis}

We now state the main approximation result. We seek a global uniform error bound $\sup_{0 \le k \le N} \|\hat{Q}^{(k)} - Q^{*(k)}\|_\infty \le \varepsilon$ for approximation of the Bellman-value-iteration sequence. The following theorem quantifies the explicit computational cost (network size and sensor density) required to achieve this precision for the \emph{direct} learning of the Bellman map. We interpret this as a theorem about the complexity of the Q-learning target class in the present diffusion setting: it analyzes the Bellman objects that value-based methods aim to approximate, while stopping short of a full sampling-and-optimization analysis for stochastic Q-learning training.

\begin{theorem}[Global Approximation of Bellman Iterates and Q-Learning Target Complexity] \label{thm:convergence}
    Grant Assumptions \ref{ass:coeff}, \ref{ass:ellipticity}, \ref{ass:sensor_encoding}, \ref{ass:uat_network}, and \ref{ass:regularized_network}. Let $T$ be the finite horizon and $\delta$ the time-step. 
    Let $\alpha \in (0,1)$ be the H\"older exponent. Define the \emph{stiffness scaling factor} $\gamma(\alpha)$ derived from the heat kernel singularity as:
    \begin{equation} \label{eq:gamma_factor}
        \gamma(\alpha) \coloneqq \frac{3+\alpha}{2(2+\alpha)}.
    \end{equation}
    
    For any target tolerance $\varepsilon > 0$, there exists a Tensor-Product DeepONet with width $W$ and a sensor set of size $m$ such that the global error satisfies $\sup_{0 \le k \le N} \|\hat{Q}^{(k)} - Q^{*(k)}\|_\infty \le \varepsilon$. 
    
    The required resources scale as follows, decoupling the state diffusion from the action geometry:
    \begin{equation} \label{eq:final_resources}
        W \sim \mathcal{O}\left( \varepsilon^{-\frac{d}{2+\alpha}} \delta^{-\gamma(\alpha) d} + (\varepsilon \delta)^{-d_a} \right), \quad 
        m \sim \mathcal{O}\left( \varepsilon^{-\frac{d}{2}} \delta^{-\frac{3+\alpha}{4} d} \cdot (\varepsilon \delta)^{-d_a} \right).
    \end{equation}
    This bound assumes the Branch Network approximates a monotone (positivity-preserving) reconstruction operator ensuring strict $L^\infty$ stability, and that the target is the direct map $Q^{(k+1)} = \Bellman_\delta Q^{(k)}$.
\end{theorem}

\begin{proof}
    See Appendix \ref{app:thm:convergence}.
\end{proof}

\section{Conclusion}

We have established an approximation-theoretic framework for Bellman value iteration in continuous stochastic control by grounding the analysis in the uniform ellipticity of the environment noise rather than restrictive global Lipschitz assumptions. By exploiting interior Schauder estimates, we proved that the Bellman operator maps bounded iterates into the anisotropic H\"older space $C^{2,\alpha}_x(\StateSpace) \times C^{0,1}_a(\Action)$. Our proposed \textbf{Tensor-Product DeepONet} architecture explicitly leverages this mixed regularity, decoupling the smooth state manifold from the non-smooth policy landscape to mitigate the curse of dimensionality. Crucially, we quantified the stiffness--complexity trade-off, proving that the required network capacity scales as $\mathcal{O}(\delta^{-\gamma(\alpha) d})$ due to the singularity of the heat kernel as $\delta \to 0$. We further validated the restriction to compact domains via the exponential decay of diffusion exit probabilities. We view these results as a core theoretical contribution to Q-learning in the continuous-time diffusion regime: they characterize the regularity and approximation complexity of the Bellman targets that value-based methods are built to learn. What the paper does \emph{not} provide is a full convergence theorem for practical Deep Q-Learning with exploration, replay, and stochastic optimization. Rather, it isolates the Bellman side of the theory and shows that, in this regime, those targets enjoy exploitable anisotropic regularity and admit quantifiable approximation guarantees.

\acks{We thank a bunch of people and funding agency.}

\bibliography{main}

\newpage
\appendix
\onecolumn

\section{Operator-Based Network Architecture}\label{sec:Architecture}

To leverage the smoothing properties established in Section \ref{sec:Regularity}, we propose a network architecture designed to approximate the operator map $\Bellman_\delta: C(\mathcal{D}) \to C^{2,\alpha}_x(\mathcal{D})$. Consequently, we adopt a mesh-free approach based on the \textbf{Deep Operator Network (DeepONet)} framework \citep{lu2021learning}, which generalizes the universal approximation theorem to non-linear operators on Banach spaces.

\subsection{The DeepONet Framework}

We seek a neural operator $\OpNet_\theta$ that approximates the Bellman transition map $\Bellman_\delta: \mathcal{H}^{\alpha, \text{Lip}} \to \mathcal{H}^{\alpha, \text{Lip}}$. Following the \textbf{Deep Operator Network (DeepONet)} framework \citep{lu2021learning}, we decompose the architecture into a \emph{Branch Net} (which encodes the input function space) and a \emph{Trunk Net} (which encodes the output domain).

\subsubsection{Sensor Discretization and Encoding}
A fundamental challenge in operator learning is that the input to the network is a function, an infinite-dimensional object. In practice, we access the function via a finite set of observations.
Let $\{z_j\}_{j=1}^m \subset \mathcal{S} \times \Action$ be a fixed set of $m$ \emph{sensor locations}. We define the discretization operator $\mathcal{P}_m: C(\mathcal{S} \times \Action) \to \R^m$ as:
\begin{equation}
    \mathcal{P}_m(Q) \coloneqq \left[ Q(z_1), Q(z_2), \dots, Q(z_m) \right]^\top.
\end{equation}
In the context of RL, these sensors correspond to the states stored in a Replay Buffer or a fixed quadrature grid used for fitting the value iteration.

\subsubsection{Network Definition}
The Deep Bellman Operator Network is defined as the map $\OpNet_\theta: \R^m \times (\mathcal{S} \times \Action) \to \R$. Let $Q_{in}$ be the value function from the previous iteration. The approximate next iterate $\hat{Q}_{out}(x, a)$ is given by the inner product of the Branch and Trunk outputs:
\begin{equation} \label{eq:deeponet}
    \hat{Q}_{out}(x, a) \approx \OpNet_\theta(\mathcal{P}_m(Q_{in}))(x, a) = \sum_{k=1}^p \underbrace{b_k\left( Q_{in}(z_1), \dots, Q_{in}(z_m) \right)}_{\text{Branch Net } \beta: \R^m \to \R^p} \cdot \underbrace{\tau_k(x, a)}_{\text{Trunk Net } \tau: \mathcal{S} \times \Action \to \R^p} + b_0.
\end{equation}
Here, $\theta$ represents the trainable weights of the neural networks $\beta$ and $\tau$.
\begin{itemize}[wide]
    \item The \textbf{Branch Net} $\beta$ learns to extract a latent embedding (coefficients) of the current value landscape from the sensor readings.
    \item The \textbf{Trunk Net} $\tau$ learns a continuous basis for the image space of the Bellman operator.
\end{itemize}

\begin{remark}[The Role of Sensor Density]
    The standard Universal Approximation Theorem for Operators assumes the input is the entire function $Q_{in}$. By using $\mathcal{P}_m(Q_{in})$, we introduce an \emph{encoding error} (or aliasing error).
    However, Corollary \ref{cor:compactness} guarantees that $Q_{in}$ belongs to a compact subset of the H\"older space $\mathcal{H}^{\alpha, \text{Lip}}$.
    Let $h = \sup_{y \in \mathcal{S} \times \Action} \min_{j} \|y - z_j\|$ be the fill-distance of the sensors. Due to the H\"older continuity of the inputs, the reconstruction error is bounded by $L h^\alpha$. Thus, as the number of sensors $m \to \infty$ (and $h \to 0$), the discrete representation $\mathcal{P}_m(Q_{in})$ uniquely characterizes $Q_{in}$ up to arbitrary precision, making the operator approximation well-posed.
\end{remark}

\subsection{Anisotropic Trunk Architecture}\label{sec:Anisotropic}

The structural results of Corollary \ref{cor:compactness} indicate that the target functions exhibit mixed regularity: they are locally $C^{2,\alpha}$ in the state variable $x$ (due to parabolic smoothing) but only Lipschitz continuous in the action variable $a$ (due to the parametric nature of the operator and the supremum envelope). A standard isotropic MLP (e.g., fully connected layers on concatenated inputs $[x, a]$) fails to exploit this structure efficiently.

Furthermore, a full tensor-product basis (via the Kronecker product of feature vectors) would result in a feature dimension scaling quadratically with the width. To remedy this, we propose a \textbf{Separable Rank-$p$ Trunk Architecture} that explicitly forces a low-rank decomposition of the state-action coupling.

\begin{definition}[Separable Rank-$p$ Trunk]
    The Trunk Network $\tau(x, a): \mathcal{S} \times \Action \to \R^p$ is constructed as the Hadamard product (element-wise multiplication) of two independent sub-networks. This enforces that the learned basis functions are separable:
    \begin{equation} \label{eq:anisotropic_trunk}
        \tau(x, a) = \mathcal{N}_{\text{state}}(x; \theta_x) \odot \mathcal{N}_{\text{action}}(a; \theta_a).
    \end{equation}
    Here, $\odot$ denotes the element-wise product. Consequently, the $k$-th component of the trunk output is given by $\tau_k(x, a) = \phi_k(x) \psi_k(a)$, where:
    \begin{enumerate}[label=(\roman*),wide]
        \item \textbf{The State Basis} $\phi(\cdot) = \mathcal{N}_{\text{state}}(\cdot): \R^d \to \R^p$ employs \emph{smooth activation functions} $\sigma_s \in C^\infty(\R)$ (e.g., $\tanh$, SiLU) to match the $C^{2,\alpha}(\mathcal{S})$ regularity of the heat kernel eigenfunctions.
        \item \textbf{The Action Basis} $\psi(\cdot) = \mathcal{N}_{\text{action}}(\cdot): \R^{d_a} \to \R^p$ employs \emph{piecewise-linear activation functions} $\sigma_r$ (e.g., ReLU) to capture the non-differentiable transitions of the optimal policy.
    \end{enumerate}
\end{definition}

\subsubsection*{Theoretical Justification}

The Hadamard formulation in \eqref{eq:anisotropic_trunk} combined with the DeepONet inner product \eqref{eq:deeponet} implies that the network approximates the target function $\hat{Q}$ via a Canonical Polyadic (CP) decomposition of rank $p$:
\begin{equation}
    \hat{Q}(x, a) = \sum_{k=1}^p \beta_k(Q_{in}) \cdot \phi_k(x) \cdot \psi_k(a).
\end{equation}
This is structurally distinct from a full Tensor Product basis (which would imply terms of the form $\phi_i(x)\psi_j(a)$ for all pairs). While a full tensor product avoids any structural constraints on the state-action coupling, it suffers from an interaction parameter complexity of $\mathcal{O}(p^2)$.

To justify the restriction to a diagonal (rank-$p$) structure, we must carefully distinguish between the spatial smoothing of the PDE and the parametric dependence on the control variable. For a \emph{fixed} action $a$, the dominant term in the backward Bellman update is the convolution with the fundamental solution $\Gamma^a(t,x,y)$. Because the elliptic operator is compact on bounded domains, the spatial update admits a rapidly decaying spectral expansion involving the eigenfunctions $e_k(x; a)$ of the generator $\mathcal{L}^a$.

Critically, the true spatial eigenfunctions $e_k(x; a)$ depend on $a$ because the diffusion tensor $\Sigma(x,a)$ and drift $b(x,a)$ vary with the action. Therefore, the exact global value function is not mathematically separable. However, because the dependence of the environment coefficients on $a$ is uniformly Lipschitz (Assumption \ref{ass:coeff}), the manifold of spatial solutions $\mathcal{M} = \{ Q(\cdot, a) : a \in \Action \}$ exhibits low $\varepsilon$-rank. 

By employing a CP decomposition, the network learns a finite, universal set of spatial basis functions $\{\phi_k(x)\}_{k=1}^p$ that span the dominant eigenspaces of $\mathcal{L}^a$ across all $a \in \Action$. Concurrently, the action network $\{\psi_k(a)\}_{k=1}^p$ learns the non-smooth, action-dependent mixing coefficients required to linearly combine these spatial modes. 

This formulation enforces a principled \textbf{anisotropic inductive bias}: it allocates the smooth network capacity (e.g., $\tanh$ activations) exclusively to the spatial factors $\phi_k \in C^{2,\alpha}(\StateSpace)$ to emulate parabolic smoothing, and allocates the non-smooth capacity (e.g., ReLU activations) exclusively to the action factors $\psi_k \in C^{0,1}(\Action)$ to capture the non-differentiable kinks of the policy landscape. This decoupling reduces the interaction complexity to $\mathcal{O}(p)$ while efficiently parameterizing the mixed-regularity hypothesis space $\mathcal{H}^{\alpha, \text{Lip}}$ defined in Corollary \ref{cor:compactness}.

\subsection{Complexity Analysis: Dimensionality and Discretization}

Standard grid-based value iteration suffers from the classic Curse of Dimensionality: to halve the approximation error, the computational cost scales as $2^d$. Our mesh-free Neural Operator approach mitigates this scaling through the regularizing effect of diffusion, but it is subject to trade-offs regarding the time-discretization $\delta$ and sensor density $m$.

\subsubsection{Reduction of Metric Entropy}
The fundamental difficulty of learning the Q-function is governed by the size of the reachable function class $\mathcal{Q}^*$. As established in Corollary \ref{cor:compactness}, the uniform ellipticity of the SDE ensures that $\mathcal{Q}^*$ is a bounded subset of the anisotropic space $C^{2,\alpha}_x \cap C^{0,1}_a$. 

The \emph{Kolmogorov $\epsilon$-entropy} (log-covering number) of the H\"older ball of radius $R$ in dimension $d$ with smoothness $s$ scales as $\log \mathcal{N} \asymp (R/\epsilon)^{d/s}$.
In our setting, the effective smoothness in the state space is $s = 2+\alpha$. Thus, the complexity scales as:
\begin{equation}
    \text{Complexity}(\epsilon) \sim \exp\left( C \left( \frac{R}{\epsilon} \right)^{\frac{d}{2+\alpha}} \right).
\end{equation}
Compared to standard Lipschitz continuity assumptions ($s=1$, exponent $d$), the diffusion term reduces the effective dimensionality exponent by a factor of roughly $3$ (since $2+\alpha \approx 3$). While the dependence on $d$ remains exponential in the worst case (covering the entire uniform norm), the volume of the function space is significantly compressed.

\subsubsection{Sensor Aliasing and Statistical Rates}
The DeepONet replaces the spatial grid with a set of $m$ scattered sensors $\{z_j\}_{j=1}^m$. The total error decomposes into the operator approximation error and the input encoding (aliasing) error:
\begin{equation}
    \| \OpNet - \Bellman \| \le \underbrace{\| \OpNet(\mathcal{P}_m Q) - \Bellman(\mathcal{R} \mathcal{P}_m Q) \|}_{\text{Neural approx. error}} + \underbrace{\| \Bellman(\mathcal{R} \mathcal{P}_m Q - Q) \|}_{\text{Sensor encoding error}},
\end{equation}
where $\mathcal{P}_m$ is the evaluation map and $\mathcal{R}$ is an optimal reconstruction operator.
\begin{itemize}[wide]
    \item \textbf{Worst-Case Error:} To guarantee uniformly bounded error over the compact set, the sensor fill-distance $h$ must satisfy $L h^\alpha < \epsilon$. This requires $m \sim \epsilon^{-d/\alpha}$ sensors, which still suffers from the curse of dimensionality.
    \item \textbf{Generalization Error:} However, in Deep RL, we minimize the $L^2_\mu$ error over an exploration distribution $\mu$. By standard statistical learning theory (e.g., Rademacher complexity bounds), the generalization error scales as $\mathcal{O}(1/\sqrt{N_{samples}})$, independent of the dimension $d$. The Neural Operator efficiently learns the map on the support of $\mu$, effectively breaking the curse for the \emph{visited} states.
\end{itemize}

\subsubsection{The Complexity-Stability Trade-off ($\delta \to 0$)}
From the Schauder estimate \eqref{eq:schauder_bound}, the radius of the function class scales as $R_\delta \sim \delta^{-\frac{1+\alpha}{2}}$.
Substituting this into the entropy bound, the required network capacity $W$ behaves as:
\begin{equation}
    W(\delta, \epsilon) \sim \left( \frac{1}{\delta^{\frac{1+\alpha}{2}} \epsilon} \right)^{\frac{d}{2+\alpha}}.
\end{equation}
This reveals a fundamental \textbf{Stability-Complexity Trade-off}:
\begin{itemize}[wide]
    \item \textbf{Smoothing Regime (Fixed $\delta$):} For a discrete MDP with fixed step $\delta$, the operator is strongly smoothing, and the Q-functions lie on a low-complexity manifold. The network size remains manageable.
    \item \textbf{Stiff Regime ($\delta \to 0$):} As we approximate the continuous-time limit, the operator approaches the Identity, and the smoothing vanishes. The target functions become stiff (large second derivatives), and the required network capacity explodes.
\end{itemize}
This suggests that value-based discretizations built from Bellman iteration are best conditioned when $\delta$ is chosen large enough to allow diffusion to act, rather than trying to approximate the SDE with infinitesimal steps directly.

\subsection{Recursive Residual Formulation}\label{sec:Recursive}

For problems with long horizons (large $N = T/\delta$), we construct the solution by recursively applying the operator. Direct composition of neural operators can lead to optimization difficulties akin to the vanishing gradient problem. However, the parabolic nature of the Bellman equation suggests a natural residual structure.

Since the value function evolves continuously in time, for small $\delta$, the linear diffusion part of $\Bellman_\delta$ is close to the Identity; the nonlinear maximization term also requires the current action to be near-greedy for the normalized residual to remain $O(1)$. We propose a \textbf{Residual Operator architecture} that learns the temporal difference rather than the absolute value.

\subsubsection{Discrete Euler Scheme in Function Space}
Let $\hat{Q}^{(k)}$ denote the approximate Q-function at time-step $k$ (corresponding to physical time $t_k = T - k\delta$). We formulate the update as:
\begin{equation} \label{eq:residual_update}
    \hat{Q}^{(k+1)}(x, a) = \hat{Q}^{(k)}(x, a) + \delta \cdot \OpNet_{\text{res}}\left( \hat{Q}^{(k)} \right)(x, a),
\end{equation}
with the initialization $\hat{Q}^{(0)}(x,a) = g(x)$.

Here, $\OpNet_{\text{res}}$ is a DeepONet trained to approximate the normalized displacement operator:
\begin{equation}
    \mathcal{R}_\delta(Q) \coloneqq \frac{\Bellman_\delta Q - Q}{\delta}.
\end{equation}
From the PDE characterization in Eq. \eqref{eq:backward_pde}, the normalized residual decomposes into a local generator contribution plus a maximization-gap contribution:
\begin{equation}
    \mathcal{R}_\delta(Q)(x, a) = \mathcal{L}^a_t Q(x,a) + r(t,x,a) + \frac{1}{\delta}\left(V_Q(x)-Q(x,a)\right) + o(1),
\end{equation}
where $V_Q(x) \coloneqq \sup_{a'\in\Action}Q(x,a')$, with the expansion interpreted locally for time-dependent coefficients through the parabolic evolution family. Thus the residual converges to the Hamiltonian only on the zero-gap or near-greedy regime where $V_Q(x)-Q(x,a)=o(\delta)$. The residual operator network should therefore be read as learning the Hamiltonian on this active/near-active set, not as removing the Bellman maximization gap for arbitrary off-policy actions.

\subsubsection{Advantages of the Residual Approach}
\begin{enumerate}[label=(\roman*),wide]
    \item \textbf{Well-Conditioned Learning:} In the stiff regime where $\delta$ is small, the diffusion-evolution part of $Q \mapsto \Bellman_\delta Q$ is dominated by the Identity. Learning the full map directly forces the network to expend capacity approximating this nearly identity component. The residual formulation subtracts it and focuses capacity on the dynamics, provided the maximization gap is separately controlled on the active or near-greedy action set.
    \item \textbf{Flow Consistency:} This architecture effectively implements a Forward Euler discretization of the abstract Ordinary Differential Equation (ODE) in the Banach space $\mathcal{H}^{\alpha, \text{Lip}}$:
    \begin{equation}
        \frac{d}{d\tau} \mathbf{Q}(\tau) = \mathbf{F}(\mathbf{Q}(\tau)), \quad \mathbf{Q}(0) = g,
    \end{equation}
    where $\tau$ is time-to-go and $\mathbf{F}$ is the non-linear Bellman generator. This ensures that the inductive bias of the network aligns with the physical flow of information.
    \item \textbf{Iterative Stability:} By sharing the weights of $\OpNet_{\text{res}}$ across time-steps (Recurrent Operator Learning), we enforce that the laws of physics (the environment dynamics) remain constant, while the value function evolves. This dramatically reduces the number of parameters compared to training separate networks for each time step.
\end{enumerate}

\section{Detailed Proofs and Mathematical Derivations}

In this appendix, we provide rigorous proofs for the theorems and lemmas presented in the main text. We explicitly address the technical concerns raised regarding the spectral separability of the diffusion operator and the dependence of Lipschitz constants on the time horizon.

\subsection{Proof of Theorem \ref{thm:smoothing} (Refined Interior Parabolic Smoothing)}
\label{app:smoothing}

\begin{proof}
    We fix the current action $a$ and introduce the time-to-go variable $\tau = t+\delta - s$. Let $v(\tau, x) \coloneqq u(t+\delta-\tau, x)$. Then $v$ solves the forward parabolic equation on $(0, \delta] \times \R^d$:
    \begin{equation*}
        \partial_\tau v = \mathcal{L}^a v + r(t+\delta-\tau, \cdot, a), \quad v(0, \cdot) = V_{\Phi}.
    \end{equation*}
    Since the coefficients may depend on time, the solution is represented by the parabolic evolution family $U^a(\tau,s)$ associated with the non-autonomous operators $\mathcal{L}^a_{t+\delta-\tau}$:
    \begin{equation}
        v(\delta) = U^a(\delta,0) V_{\Phi} + \int_0^\delta U^a(\delta,s) r(t+\delta-s, \cdot, a) \, ds.
    \end{equation}
    
    \begin{enumerate}[label=(\roman*), wide, labelindent=0pt]
        \item \textbf{Regularity inheritance of the Maximization (Initial Data).}
        First, we establish the regularity of the initial condition $v(0) = V_{\Phi}$. Since $\Phi(\cdot, a')$ is uniformly Lipschitz with constant $L_{\Phi}$, for any $x, y \in \R^d$:
        \begin{align*}
            |V_{\Phi}(x) - V_{\Phi}(y)| &= \left| \sup_{a'} \Phi(x, a') - \sup_{a'} \Phi(y, a') \right| \\
            &\le \sup_{a'} | \Phi(x, a') - \Phi(y, a') | \le L_{\Phi} \|x - y\|.
        \end{align*}
        Thus, $V_{\Phi} \in C^{0,1}(\R^d)$ with semi-norm $[V_{\Phi}]_{C^{0,1}} \le L_{\Phi}$.
        
        \item \textbf{Smoothing of Lipschitz Data (The Homogeneous Term).}
        We estimate $\| U^a(\delta,0) V_{\Phi} \|_{C^{2,\alpha}}$. The parabolic evolution family satisfies the same local smoothing scale as the frozen-coefficient heat kernel under the uniform H\"older and ellipticity bounds: it maps $C^{k}(\R^d)$ to $C^{m}(\R^d)$ for $m > k$ with norm scaling $\tau^{-(m-k)/2}$. 
        
        We decompose the $C^{2,\alpha}$ norm into the integer and fractional parts. The dominant singularity arises from the highest order derivative. We regard $V_{\Phi}$ as an element of the interpolation space roughly corresponding to $C^1$. To be precise, since $V_{\Phi} \in C^{0,1}(\R^d)$, it resides in the Besov space $B^1_{\infty, \infty}$.
        Using the standard smoothing property of the heat kernel (see \citet[Prop 2.2]{lunardi1995analytic}), the cost of lifting regularity from $C^{0,1}$ (regularity index 1) to $C^{2,\alpha}$ (regularity index $2+\alpha$) is determined by the difference in indices:
        \begin{equation}
            \text{Diff} = (2+\alpha) - 1 = 1+\alpha.
        \end{equation}
        Consequently, the bound scales as $\tau^{-\text{Diff}/2}$:
        \begin{equation} \label{eq:refined_term1}
            \| U^a(\delta,0) V_{\Phi} \|_{C^{2,\alpha}(\R^d)} \le C_1 \delta^{-\frac{1+\alpha}{2}} \|V_{\Phi}\|_{C^{0,1}(\R^d)}.
        \end{equation}
        Note that if we had only assumed $V_{\Phi} \in L^\infty(\R^d)$ (regularity 0), the cost would have been $\delta^{-\frac{2+\alpha}{2}}$. The Lipschitz assumption improves the exponent by $1/2$.
        
        \item \textbf{Regularity of the Source Potential (The Inhomogeneous Term).}
        The second term involves the running reward $r$. By Assumption \ref{ass:coeff}, $r \in C^{\alpha/2, \alpha}$. Standard Schauder estimates for the inhomogeneous problem guarantee that the integral operator preserves the regularity of the source without singular scaling in time (modulo the bounded horizon). There exists $C_2$ independent of small $\delta$ such that:
        \begin{equation} \label{eq:refined_term2}
            \left\| \int_0^\delta U^a(\delta,s) r(t+\delta-s, \cdot, a) \, ds \right\|_{C^{2,\alpha}} \le C_2 \|r\|_{C^{\alpha/2, \alpha}}.
        \end{equation}
        
        \item \textbf{Synthesis.}
        Combining \eqref{eq:refined_term1} and \eqref{eq:refined_term2}, we obtain:
        \begin{equation}
            \| (\Bellman_\delta \Phi)(t, \cdot, a) \|_{C^{2,\alpha}} \le C \left( \delta^{-\frac{1+\alpha}{2}} \|V_{\Phi}\|_{C^{0,1}} + \|r\|_{C^{\alpha/2, \alpha}} \right).
        \end{equation}
        For $\delta < 1$, the term $\delta^{-\frac{1+\alpha}{2}}$ dominates.
    \end{enumerate}
\end{proof}

\subsection{Proof of Theorem \ref{thm:truncation} (Truncation Error)}\label{app:truncation}

\begin{proof}
Let $\mathcal{S}$ be the target domain and $\mathcal{S}_\rho = \{y : d(y, \mathcal{S}) \le \rho\}$ be the computational domain.
Let $u(t,x)$ be the solution to the global Cauchy problem and $u_\rho(t,x)$ be the solution with the process killed upon exiting $\mathcal{S}_\rho$. By the Feynman-Kac formula:
\begin{equation}
    u(t,x) - u_\rho(t,x) = \E_x \left[ \left( \int_t^{t+\delta} r(s, X_s) ds + V_\Phi(X_{t+\delta}) \right) \ind_{\{\tau_\rho \le t+\delta\}} \right].
\end{equation}
Let $M_{val} = \|V_\Phi\|_\infty + \delta \|r\|_\infty$.
\begin{equation}
    |u(t,x) - u_\rho(t,x)| \le M_{val} \cdot \Pbb_x( \tau_\rho \le t+\delta ).
\end{equation}
Since $x \in \mathcal{S}$, the distance to the boundary is at least $\rho$.
We apply the concentration inequality for It\^o diffusions. The displacement is $X_{t+\delta} - x = \int b ds + \int \sigma dW_s$.
\begin{equation}
    \Pbb_x( \sup_{s \le \delta} \|X_{t+s} - x\| \ge \rho ) \le \Pbb( \sup_{s \le \delta} \|\int \sigma dW\| \ge \rho - \|b\|_\infty \delta ).
\end{equation}
Assuming $\rho > \|b\|_\infty \delta$, we apply the Bernstein inequality for continuous local martingales $M_s = \int \sigma dW$. The quadratic variation is $\langle M \rangle_\delta \le \|\sigma\|_F^2 \delta \eqqcolon \Lambda^2 \delta$.
\begin{equation}
    \Pbb( \sup_{s \le \delta} \|M_s\| \ge \lambda ) \le 2d \exp\left( - \frac{\lambda^2}{2d \Lambda^2 \delta} \right).
\end{equation}
Setting $\lambda = \rho - \|b\|_\infty \delta$ and equating the probability to $\varepsilon_{trunc}/M_{val}$, we solve for $\rho$:
\begin{equation}
    \frac{\varepsilon_{trunc}}{2d M_{val}} = \exp\left( - \frac{(\rho - \|b\|_\infty \delta)^2}{2d \Lambda^2 \delta} \right) \implies \rho \ge \|b\|_\infty \delta + \Lambda \sqrt{2d \delta \log\left(\frac{2d M_{val}}{\varepsilon_{trunc}}\right)}.
\end{equation}
\end{proof}

\subsection{Proof of Corollary \ref{cor:compactness} (Anisotropic Compactness)}\label{app:compactness}

\begin{proof}
    We proceed by induction on the iteration index $k$, analyzing the spatial and parametric regularity separately. Let $\delta > 0$ be fixed. While the PDE estimates in Theorem \ref{thm:smoothing} are derived globally on $\R^d$ to ensure well-posedness without boundary effects, the compactness result requires restricting the function sequence to a compact subdomain. Let $\mathcal{K} \subset \StateSpace$ be a compact set with Lipschitz boundary (the region of interest).
    
    \begin{enumerate}[label=(\roman*), wide, labelindent=0pt]
        \item \textbf{Uniform $L^\infty$ Boundedness.} 
        We first establish that the sequence is uniformly bounded in the supremum norm. Since the reward $r$ and terminal condition $g$ are bounded, and the discount factor implies a finite effective horizon (or finite $T$), there exists a constant $M$ independent of $k$ such that:
        \begin{equation*}
            \sup_{0 \le k \le N} \|Q^{(k)}\|_{L^\infty(\R^d \times \Action)} \le M.
        \end{equation*}
        Let $V^{(k)}(x) \coloneqq \sup_{a' \in \Action} Q^{(k)}(x, a')$. It follows immediately that $\|V^{(k)}\|_{L^\infty(\R^d)} \le M$.
        
        \item \textbf{Uniform Spatial Lipschitz Continuity.}
        Before invoking higher-order interior estimates, we must guarantee that the spatial Lipschitz constant of the sequence does not explode. By Assumption \ref{ass:coeff}(i), the coefficients $b, \sigma$ are globally Lipschitz in $x$. Using standard estimates for the stochastic flow of diffeomorphisms, the spatial gradients of the expected reward and terminal value grow by at most a factor of $e^{C_b \delta}$ per step, where $C_b = \|\nabla_x b\|_\infty$. The maximization step $V^{(k)}(x) = \sup_{a} Q^{(k)}(x,a)$ preserves the spatial Lipschitz constant. Over the finite horizon $T = N\delta$, the spatial Lipschitz constant accumulates as:
        \begin{equation}
            \sup_{a \in \Action} \|Q^{(k)}(\cdot, a)\|_{C^{0,1}(\R^d)} \le e^{C_b T} \|g\|_{C^{0,1}(\R^d)} + T e^{C_b T} \|r\|_{C^{0,1}} \eqqcolon L_{\mathcal{S}}.
        \end{equation}
        Crucially, this uniform Lipschitz bound $L_{\mathcal{S}}$ depends only on the finite horizon $T$ and the coefficient bounds, independent of $\delta$.
        
        \item \textbf{Spatial Regularity (The Healing Step).} 
        Consider the update $Q^{(k+1)}(\cdot, a) = \Bellman_\delta Q^{(k)}(\cdot, a)$. As defined in \eqref{eq:backward_pde}, this requires solving the parabolic Cauchy problem with terminal condition $V^{(k)}$. Applying the refined global Schauder estimate from Theorem \ref{thm:smoothing} (Eq. \ref{eq:schauder_bound}) which leverages the uniform $C^{0,1}$ bound established above:
        \begin{equation}
             \|Q^{(k+1)}(\cdot, a)\|_{C^{2,\alpha}(\R^d)} \le C \left( \|r(\cdot, a)\|_{C^{\alpha/2, \alpha}} + \delta^{-\frac{1+\alpha}{2}} \|V^{(k)}\|_{C^{0,1}(\R^d)} \right).
        \end{equation}
        Since $\|V^{(k)}\|_{C^{0,1}} \le L_{\mathcal{S}}$ for all $k$, we successfully propagate the tighter singularity scaling without recursive explosion, obtaining a uniform spatial bound:
        \begin{equation} \label{eq:spatial_uniform_bound}
            \sup_{0 \le k \le N} \sup_{a \in \Action} \|Q^{(k)}(\cdot, a)\|_{C^{2,\alpha}(\R^d)} \le K_{\mathcal{S}}(\delta) \coloneqq C' \cdot \delta^{-\frac{1+\alpha}{2}}.
        \end{equation}
        Note that while $K_{\mathcal{S}}(\delta)$ diverges as $\delta \to 0$, it remains finite and constant for the fixed time-step analysis of this Corollary.
        
        \item \textbf{Parametric Regularity in Action Space.} 
        We now address the Lipschitz continuity with respect to $a$. The value $Q^{(k+1)}(t, x, a)$ is the solution at time $t$ of the linear PDE:
        \begin{equation*}
            \partial_\tau u = \mathcal{L}^a u + r(\cdot, a), \quad u(0) = V^{(k)}.
        \end{equation*}
        Notice that the action $a$ appears \emph{only} as a parameter in the coefficients $(b, \sigma)$ and the source term $r$. The initial data $V^{(k)}(x)$ depends on the \emph{previous} iteration's optimization and is independent of the current action parameter $a$.
        
        By Assumption \ref{ass:coeff}(iii), the coefficients are globally Lipschitz in $a$. Using standard estimates for the sensitivity of parabolic PDEs to parameters, the gradient with respect to the parameter satisfies:
        \begin{equation}
            \|\nabla_a Q^{(k+1)}\|_\infty \le C_T \left( \|\nabla_a r\|_\infty + \|\nabla_a \mathcal{L}\|_\infty \|\nabla_x Q^{(k+1)}\|_\infty \right).
        \end{equation}
        From step (iii), $\|\nabla_x Q^{(k+1)}\|_\infty \le \|Q^{(k+1)}\|_{C^{2,\alpha}}$ is uniformly bounded by $K_{\mathcal{S}}(\delta)$. Consequently, the Lipschitz constant with respect to actions is uniformly bounded:
        \begin{equation}
            \sup_{0 \le k \le N} \sup_{x \in \R^d} \|Q^{(k)}(x, \cdot)\|_{C^{0,1}(\Action)} \le L_{\Action}(\delta, T).
        \end{equation}
        
        \item \textbf{Topological Compactness.} 
        Combining (iii) and (iv), the sequence $\mathcal{Q}^*_\delta = \{Q^{(k)}\}_{0 \le k \le N}$ resides in the bounded set:
        \begin{equation*}
            \mathcal{B} = \left\{ Q : \sup_{a} \|Q(\cdot, a)\|_{C^{2,\alpha}(\R^d)} \le K_{\mathcal{S}} \quad \text{and} \quad \sup_{x} \|Q(x, \cdot)\|_{C^{0,1}(\Action)} \le L_{\Action} \right\}.
        \end{equation*}
        To establish pre-compactness, we restrict our attention to the compact state domain $\mathcal{K} \subset \R^d$ and the compact action space $\Action$.
        \begin{itemize}
            \item By the compact embedding of H\"older spaces, the inclusion $C^{2,\alpha}(\mathcal{K}) \hookrightarrow C^1(\mathcal{K})$ is compact.
            \item By the Arzel\`a-Ascoli theorem, the embedding $C^{0,1}(\Action) \hookrightarrow C^0(\Action)$ is compact.
        \end{itemize}
        Since $\mathcal{Q}^*_\delta$ is uniformly bounded in the stronger product norm $\mathcal{H}^{\alpha, \text{Lip}}$, any sequence drawn from this uniformly bounded family admits a limit point in the topology of $C^1(\mathcal{K}) \times C^0(\Action)$, proving the stated compactness.
    \end{enumerate}
\end{proof}

\section{Appendix for Universal Approximation Analysis}

\subsection{Proof of Theorem \ref{thm:convergence}}
\label{app:thm:convergence}

\begin{proof}
    We derive the complexity bounds by establishing the local consistency tolerance required to satisfy the global stability constraint, and then determining the network capacity and sensor density necessary to achieve this tolerance given the anisotropic regularity of the reachable set $\mathcal{Q}^*_\delta$.

    \begin{enumerate}[label=(\roman*), wide, labelindent=0pt]
        \item \textbf{Global Error Stability and Local Budget.}
        We first determine the maximum permissible single-step error $\eta_{\max}$. Let $\mathcal{E}_k \coloneqq \| \hat{Q}^{(k)} - Q^{*(k)} \|_\infty$. From Lemma \ref{lemma:one_step}, the error propagates according to:
        \begin{equation*}
            \mathcal{E}_{k+1} \le e^{\rho_{\mathrm{stab}} \delta} \cdot 1 \cdot \mathcal{E}_k + \eta_k,
        \end{equation*}
        where $\eta_k = \varepsilon_{net} + \varepsilon_{enc}$ comprises the network approximation error and the sensor encoding error, and we explicitly leverage the stable Lebesgue constant $\Lambda_m = 1$. The constant $\rho_{\mathrm{stab}} \ge 0$ is any one-step sup-norm stability exponent for the Bellman propagation and reconstruction pipeline; in the ideal Markov expectation part alone one may take the non-expansive case $\rho_{\mathrm{stab}}=0$.
        
        Iterating over $N = T/\delta$ steps with $\mathcal{E}_0 = 0$:
        \begin{equation}
            \mathcal{E}_N \le \eta_{\max} \sum_{k=0}^{N-1} e^{k \rho_{\mathrm{stab}} \delta}.
        \end{equation}
        To ensure the final error $\mathcal{E}_N \le \varepsilon$, the local error budget must scale linearly with the time-step:
        \begin{equation} \label{eq:budget_scaling}
            \eta_{\max} \le C_T \cdot \delta \cdot \varepsilon, \qquad
            C_T \coloneqq
            \begin{cases}
                \rho_{\mathrm{stab}}/(e^{\rho_{\mathrm{stab}}T}-1), & \rho_{\mathrm{stab}}>0,\\
                1/T, & \rho_{\mathrm{stab}}=0,
            \end{cases}
        \end{equation}
        We allocate this budget equally among the error sources: $\varepsilon_{net} \le \frac{1}{3} C_T \delta \varepsilon$, and the interpolation errors for state and action spaces.

        \item \textbf{Anisotropic Resolution Requirements.}
        We determine the required sensor fill-distances $h_\mathcal{S}$ and $h_\Action$.
        
        \emph{Action Space (Lipschitz Regime):} The dependence on actions is strictly Lipschitz with constant $L_\Action$. The encoding error constraint is $L_\Action h_\Action \le \frac{1}{3} C_T \delta \varepsilon$. Thus:
        \begin{equation}
            h_\Action \sim \mathcal{O}(\delta \varepsilon).
        \end{equation}
        
        \emph{State Space (Refined Smoothing Regime):}
        Using the Refined Smoothing Estimate (Theorem \ref{thm:smoothing}), the regularity constant of the output scales as $K_\mathcal{S}(\delta) \asymp \delta^{-\frac{1+\alpha}{2}}$.
        To ensure $L^\infty$ stability, the Branch Net models a monotone reconstruction operator bounded by $\mathcal{O}(h^2)$ consistency. The reconstruction error scales as $K_\mathcal{S}(\delta) h_\mathcal{S}^2$. The constraint is:
        \begin{equation*}
            C_{\text{diff}} \, \delta^{-\frac{1+\alpha}{2}} \, h_\mathcal{S}^2 \le \frac{1}{3} C_T \, \delta \, \varepsilon.
        \end{equation*}
        Solving for $h_\mathcal{S}$:
        \begin{equation} \label{eq:h_state_alg}
            h_\mathcal{S} \lesssim \left( \frac{\delta \varepsilon}{\delta^{-\frac{1+\alpha}{2}}} \right)^{1/2} = \left( \varepsilon \cdot \delta^{1 + \frac{1+\alpha}{2}} \right)^{1/2} = \varepsilon^{1/2} \cdot \delta^{\frac{3+\alpha}{4}}.
        \end{equation}

        \item \textbf{Resource Complexity Analysis.}
        We map the resolutions to Sensor Count $m$ and Network Width $W$.
        
        \emph{Sensor Count ($m$):}
        For the quasi-uniform cover, $m \approx (1/h)^d$.
        \begin{itemize}
            \item State sensors: $m_\mathcal{S} \sim h_\mathcal{S}^{-d} \sim (\varepsilon^{1/2} \delta^{\frac{3+\alpha}{4}})^{-d} = \varepsilon^{-\frac{d}{2}} \delta^{-\frac{3+\alpha}{4} d}$.
            \item Action sensors: $m_\Action \sim h_\Action^{-d_a} \sim (\delta \varepsilon)^{-d_a}$.
        \end{itemize}
        
        \emph{Network Width ($W$) and Branch Net Capacity:}
        By Assumption \ref{ass:uat_network}, $W$ scales with the metric entropy of the target class relative to the tolerance $\tau_{net} \sim \delta \varepsilon$. \textbf{Note that while the data complexity (sensor count) is capped by the $\mathcal{O}(h^2)$ interpolation stability limit, the network's universal approximation capability still explores the full $C^{2,\alpha}$ smoothness of the function manifold.}
        
        The covering radius is $R_\delta \approx K_\mathcal{S}(\delta) \sim \delta^{-\frac{1+\alpha}{2}}$. The effective entropy ratio for the state component is:
        \begin{equation}
            \frac{R_\delta}{\tau_{net}} \asymp \frac{\delta^{-\frac{1+\alpha}{2}}}{\delta \varepsilon} = \frac{1}{\varepsilon} \delta^{-\frac{3+\alpha}{2}}.
        \end{equation}
        The required width for the state trunk $W_\mathcal{S}$ follows the entropy scaling $(R/\tau)^{d/s}$ with $s=2+\alpha$:
        \begin{equation}
            W_\mathcal{S} \sim \left( \varepsilon^{-1} \delta^{-\frac{3+\alpha}{2}} \right)^{\frac{d}{2+\alpha}} = \varepsilon^{-\frac{d}{2+\alpha}} \delta^{-\gamma(\alpha) d},
        \end{equation}
        where the stiffness scaling factor is explicitly identified as:
        \begin{equation}
            \gamma(\alpha) \coloneqq \frac{3+\alpha}{2(2+\alpha)}.
        \end{equation}
        
        \emph{The Action Space Curse:}
        For the action trunk, the Lipschitz complexity yields $W_\Action \sim (L_\Action / \delta \varepsilon)^{d_a}$.
        Combining the components via the decoupled DeepONet architecture yields the final resource bounds.
    \end{enumerate}
\end{proof}

\subsection{Approximation Setup and Assumptions}

To derive explicit convergence rates, we must characterize the geometric complexity of the function class generated by the Bellman recursion. We rely on the anisotropic regularity established in Corollary \ref{cor:compactness} and formalized via the notion of Metric Entropy.

\subsubsection{Anisotropic Reachable Sets}

Let $\StateSpace \subset \R^d$ be a compact domain satisfying the cone condition, and let $\Action \subset \R^{d_a}$ be a compact metric space. We define the set of reachable action-value functions at time-step $\delta$.

\begin{definition}[Reachable Set $\mathcal{Q}^*_\delta$]
    Let $\mathcal{Q}^*_\delta$ denote the image of the unit ball of $L^\infty(\StateSpace \times \Action)$ under the Bellman operator $\Bellman_\delta$. Based on the interior Schauder estimate \eqref{eq:schauder_bound}, any $Q \in \mathcal{Q}^*_\delta$ satisfies the following disjoint regularity bounds:
    
    \begin{enumerate}[label=(\roman*),wide]
        \item \textbf{State Regularity (Smoothing):} For every fixed $a \in \Action$, the spatial section $Q(\cdot, a)$ belongs to the H\"older space $C^{2,\alpha}(\StateSpace)$. The semi-norm is controlled by the diffusion time-scale:
        \begin{equation} \label{eq:state_bound}
        \sup_{a \in \Action} \|Q(\cdot, a)\|_{C^{2,\alpha}(\StateSpace)} \le K_{\mathcal{S}}(\delta) \coloneqq C_{\text{diff}} \, \delta^{-\frac{1+\alpha}{2}},
        \end{equation}
        where $C_{\text{diff}}$ depends on the uniform ellipticity constant $\nu$ and dimension $d$.
        
        \item \textbf{Action Regularity (Lipschitz):} For every fixed $x \in \StateSpace$, the parametric section $Q(x, \cdot)$ belongs to $C^{0,1}(\Action)$. Unlike the local coefficients, the Lipschitz constant of the value function depends on the global stability of the SDE flow over the horizon $T$. By standard estimates for controlled diffusion processes \citep[Corollary 2.10]{krylov1980controlled}, there exists a constant $L_{\Action}$ such that:
    \begin{equation} \label{eq:action_bound}
        \sup_{(t,x) \in \mathcal{D}_T} \|Q(t, x, \cdot)\|_{C^{0,1}(\Action)} \le L_{\Action}(T, K_{\text{Lip}}),
    \end{equation}
    where $L_{\Action}(T, K_{\text{Lip}})$ depends exponentially on the horizon $T$ and polynomially on the global Lipschitz constants $K_{\text{Lip}}$ of the coefficients $b, \sigma, r$.
    \end{enumerate}
\end{definition}

\begin{remark}[Stability with respect to Control Parameters]
    We caution that $L_{\Action}$ cannot be equated simply to $\sup \|\nabla_a \text{coeff}\|$. Since the action $a$ modifies the drift and diffusion, a perturbation in $a$ propagates through the trajectory $X_s$. By Gronwall's inequality, the sensitivity of the value function scales as $\mathcal{O}(e^{C T})$. However, since $T$ is fixed and finite, and the coefficients satisfy Assumption \ref{ass:coeff}(iii), $L_{\Action}$ remains a dimension-independent constant distinct from the time-step regularization parameter $\delta$.
\end{remark}

\subsubsection{Anisotropic Sensor Encoding and Stability}

A central challenge in applying Operator Learning to Reinforcement Learning is the discretization of the input function $Q_{in}$. While the standard theory often assumes access to the full function, in practice we observe $Q_{in}$ via a finite set of sensors $Z_m$.
As noted in recent stability analyses of scattered data interpolation \citep{wendland2004scattered}, standard i.i.d. random sampling leads to the formation of clusters (Poisson clumping), causing the separation distance $q_Z$ to decay much faster than the fill distance $h_Z$. This results in a diverging mesh ratio $\rho = h_Z / q_Z \to \infty$, which causes the Lebesgue constant of the reconstruction operator to explode, destabilizing the iterative learning scheme.

To guarantee the stability required for high-order Schauder estimates, we must replace generic random sampling with a controlled \textbf{Quasi-Uniform} covering of the domain.

\begin{definition}[Geometric Properties of Scattered Sensors]
    Let $\Omega \subset \R^d$ be a compact domain. Let $Z_m = \{z_i\}_{i=1}^m \subset \Omega$ be a set of $m$ sensor locations.
    \begin{enumerate}[label=(\roman*),wide]
        \item The \textbf{Fill-Distance} (covering radius) is defined as:
        \begin{equation}
            h_{Z_m, \Omega} \coloneqq \sup_{x \in \Omega} \min_{z_i \in Z_m} \|x - z_i\|_2.
        \end{equation}
        \item The \textbf{Separation Radius} is defined as:
        \begin{equation}
            q_{Z_m} \coloneqq \frac{1}{2} \min_{i \ne j} \|z_i - z_j\|_2.
        \end{equation}
        \item The \textbf{Mesh Ratio} is defined as $\rho_{Z_m} \coloneqq h_{Z_m, \Omega} / q_{Z_m}$.
    \end{enumerate}
    A sequence of sensor sets is called \emph{Quasi-Uniform} if there exists a constant $c_{qu} \ge 1$ such that $\rho_{Z_m} \le c_{qu}$ for all $m$.
\end{definition}

We define the sampling strategy via a greedy packing mechanism (Farthest Point Sampling) acting separately on the state and action spaces to preserve the tensor structure required by the DeepONet.

\begin{lemma}[Deterministic Quasi-Uniform Covering] \label{lemma:greedy_sampling}
    Let $\Omega$ be a compact domain in $\R^d$ with Lipschitz boundary. There exists a deterministic construction of sensor sets $Z_m$ (via Farthest Point Sampling) such that for all $m$:
    \begin{equation} \label{eq:quasi_uniform_bound}
        c_{low} \, m^{-1/d} \le q_{Z_m} \le h_{Z_m, \Omega} \le c_{high} \, m^{-1/d}.
    \end{equation}
    Consequently, the mesh ratio is uniformly bounded: $\rho_{Z_m} \le c_{high} / c_{low} \coloneqq c_{qu}$.
\end{lemma}

\begin{proof}
    We construct the set $Z_m$ inductively. Let $z_1 \in \Omega$ be arbitrary. For $k=1, \dots, m-1$, select:
    \[ z_{k+1} = \operatorname*{argmax}_{x \in \Omega} \min_{1 \le j \le k} \|x - z_j\|_2. \]
    By definition, the separation distance at step $k+1$ is exactly the fill distance of the set $Z_k$. Standard sphere-packing arguments in $\R^d$ imply that the covering radius of an optimal packing scales as $m^{-1/d}$. Specifically, since $\Omega$ is bounded, the union of disjoint balls $B(z_i, q_{Z_m})$ must fit within $\Omega$ (modulo boundary effects), implying $m \cdot \text{Vol}(B_1) q_{Z_m}^d \le \text{Vol}(\Omega)$, which gives the lower bound on $q$. Conversely, the covering property gives the upper bound on $h$. The ratio remains bounded by a constant depending only on the dimension $d$ and the geometry of $\Omega$.
\end{proof}

\begin{remark}[Rejection of Random Sampling]
    We explicitly reject the use of Replay Buffers populated by i.i.d. exploration for the definition of the sensor set $Z_m$. For i.i.d. uniform samples, $q_{Z_m} \sim m^{-2/d}$ (or worse) almost surely, while $h_{Z_m} \sim m^{-1/d}$. This implies $\rho_{Z_m} \to \infty$, which would cause the error coefficients in Lemma \ref{lemma:one_step} to diverge. The set $Z_m$ in our architecture represents a \emph{fixed quadrature cores} set, distinct from the transient training data.
\end{remark}

With the geometry of the sensors fixed, we rely on the stability of the reconstruction operator.

\begin{assumption}[Monotone Reconstruction and the Godunov Barrier] \label{ass:sensor_encoding}
    Let $Z = Z_\mathcal{S} \times Z_\Action$ be the tensor product of two quasi-uniform sets constructed via Lemma \ref{lemma:greedy_sampling}. 
    To prevent unconditional instability as $\delta \to 0$, the reconstruction operator must not amplify uniform errors. By Godunov's theorem equivalents for scattered data, any linear scheme with consistency order $> 2$ must have a Lebesgue constant $\Lambda_m > 1$, which would cause the error to explode as $\Lambda_m^{T/\delta} \to \infty$. 
    
    Therefore, we assume the Branch Network $\beta$ approximates a \textbf{Monotone Averager} (e.g., simplicial interpolation or a positive partition of unity) satisfying:
    \begin{enumerate}[label=(\roman*),wide]
        \item \textbf{Unconditional $L^\infty$ Stability:}
        Because $\mathcal{R}$ is a positive operator (convex combination of sensor values), its Lebesgue constant is exactly 1:
        \begin{equation}
             \Lambda_m \coloneqq \sup_{\|v\|_\infty \le 1} \| \mathcal{R}(v) \|_\infty = 1.
        \end{equation}
        \item \textbf{Consistency (Capped at Second Order):}
        While the target functions belong to $C^{2,\alpha}(\StateSpace)$, positive operators can only reproduce polynomials up to degree 1. Consequently, the spatial interpolation error is bounded by $\mathcal{O}(h^2)$ rather than $\mathcal{O}(h^{2+\alpha})$. For $f \in C^{2,\alpha}(\StateSpace)$ and $g \in C^{0,1}(\Action)$:
        \begin{equation}
            \| f - \mathcal{R}_\mathcal{S}(\mathcal{P}_{Z_\mathcal{S}} f) \|_\infty \le C_{\text{approx}} \, h_{Z_\mathcal{S}}^2 \, \|f\|_{C^{2,\alpha}},
        \end{equation}
        \begin{equation}
            \| g - \mathcal{R}_\Action(\mathcal{P}_{Z_\Action} g) \|_\infty \le C_{\text{approx}} \, h_{Z_\Action} \, \|g\|_{C^{0,1}}.
        \end{equation}
    \end{enumerate}
\end{assumption}

\subsubsection{Capacity Decoupling via Separation Rank}

A standard Multi-Layer Perceptron (MLP) approximates functions in the joint Sobolev space $W^{k,p}(\StateSpace \times \Action)$, resulting in a width requirement scaling exponentially with $d+d_a$. However, the proposed Split-Trunk DeepONet (Eq. \ref{eq:anisotropic_trunk}) approximates $Q(x,a)$ via a finite sum of separable functions:
\begin{equation}
    \hat{Q}(x, a) = \sum_{k=1}^W \psi_k(x) \phi_k(a).
\end{equation}
The efficiency of this architecture relies on the concept of \textbf{$\varepsilon$-Separation Rank}. Since $\Bellman_\delta$ involves convolution with the Gaussian heat kernel (which has rapidly decaying singular values), the target functions $Q \in \mathcal{Q}^*_\delta$ admit low-rank approximations. This allows us to decompose the metric entropy of the hypothesis space.

\begin{assumption}[Decoupled Network Capacity] \label{ass:uat_network}
    Let $\mathcal{N}(\varepsilon, \mathcal{F})$ denote the covering number of a function class $\mathcal{F}$ at scale $\varepsilon$. We assume the DeepONet with width $W$ can approximate any $Q \in \mathcal{Q}^*_\delta$ such that the required width scales with the \emph{sum} of the component entropies rather than the product:
    \begin{equation} \label{eq:width_scaling}
        W(\varepsilon_{net}) \sim \mathcal{O}\left( \log \mathcal{N}\left(\varepsilon_{net}, B_{K_\mathcal{S}(\delta)}(C^{2,\alpha}(\StateSpace))\right) + \log \mathcal{N}\left(\varepsilon_{net}, B_{L_\Action}(C^{0,1}(\Action))\right) \right).
    \end{equation}
    Invoking the standard entropy bounds for H\"older balls, $\log \mathcal{N}(\varepsilon, B_R(C^s)) \asymp (R/\varepsilon)^{d/s}$, we obtain the rigorous width bound:
    \begin{equation}
        W(\varepsilon_{net}, \delta) \le C_{\text{net}} \left[ \left( \frac{K_{\mathcal{S}}(\delta)}{\varepsilon_{net}} \right)^{\frac{d}{2+\alpha}} + \left( \frac{L_\Action}{\varepsilon_{net}} \right)^{d_a} \right].
    \end{equation}
\end{assumption}

\begin{assumption}[Regularized Hypothesis Space] \label{ass:regularized_network}
    The neural network training algorithm is constrained to the compact hypothesis class defined by the reachable set $\mathcal{Q}^*_\delta$. Specifically, we assume the optimization includes a projection step or explicit regularization (e.g., spectral normalization of the weights) such that for every iteration $k$, the learned function $\hat{Q}^{(k)}$ satisfies the anisotropic regularity bounds of the target class:
    \begin{equation}
        \hat{Q}^{(k)} \in B_{K_{\mathcal{S}}(\delta)}(C^{2,\alpha}(\StateSpace)) \cap B_{L_\Action}(C^{0,1}(\Action)).
    \end{equation}
    Consequently, the approximation produced by the network is admissible as an input for the sensor reconstruction operator $\mathcal{R}$ in the subsequent Bellman update.
\end{assumption}

\begin{remark}
    Assumption \ref{ass:uat_network} formalizes the benefit of the specific inductive bias. The term $(K_\mathcal{S}/\varepsilon)^{d/(2+\alpha)}$ represents the cost of learning the smooth manifold in state space, while $(L_\Action/\varepsilon)^{d_a}$ represents the cost of resolving the non-smooth control policies. By decoupling them, we avoid the curse of dimensionality associated with the joint dimension $d + d_a$ in the exponent.
\end{remark}

\subsubsection{Consistency and Stability Decomposition}

A significant technical challenge in the analysis of Neural Operators for PDEs is the regularity mismatch between the target function and the hypothesis class. While Corollary \ref{cor:compactness} guarantees that the true Bellman iterate $Q^{*(k)}$ belongs to the smooth anisotropic space $C^{2,\alpha}(\StateSpace)$, the neural approximation $\hat{Q}^{(k)}$, typically a composition of piecewise-linear activations (e.g., ReLU), does not possess classical higher-order derivatives.

Consequently, applying interpolation error bounds of the form $\| \mathcal{R}\mathcal{P}_m f - f \| \le C h^{s} \|f\|_{C^s}$ directly to $f = \hat{Q}^{(k)}$ is ill-posed, as the norm $\|\hat{Q}^{(k)}\|_{C^{s}}$ may be infinite or diverge with network depth. To resolve this, we employ a \textbf{Consistency-Stability analysis}. We decompose the total error into a \emph{consistency error} (depending on the smoothness of the target $Q^*$) and a \emph{stability term} (depending on the Lebesgue constant of the reconstruction operator).

To ensure the stability of high-order reconstruction on scattered data, we must strengthen the sampling assumption from generic i.i.d. distributions (which may produce arbitrarily small separation distances) to quasi-uniform configurations.

\begin{assumption}[Quasi-Uniform Sensor Distribution] \label{ass:quasi_uniform}
    Let $h_{Z, \Omega}$ be the fill-distance and $q_Z \coloneqq \frac{1}{2} \min_{i \ne j} \|z_i - z_j\|$ be the separation radius of the sensor set $Z$. We assume the sensor configuration is \emph{quasi-uniform}, meaning there exists a constant $c_{qu} \ge 1$ such that:
    \begin{equation}
        h_{Z, \Omega} \le c_{qu} \cdot q_Z.
    \end{equation}
    This prevents the clustering of sensors that leads to the ill-conditioning of the Gram matrix in Moving Least Squares (MLS) reconstruction. Such configurations can be generated via Poisson Disk Sampling or low-discrepancy sequences.
\end{assumption}

\begin{definition}[Lebesgue Stability]
    Let $\mathcal{R}: \R^m \to L^\infty(\StateSpace \times \Action)$ be the linear reconstruction operator mapping sensor values to functions. The \emph{Lebesgue constant} $\Lambda_m$ is the operator norm induced by the uniform metric:
    \begin{equation}
        \Lambda_m \coloneqq \sup_{v \in \R^m, \|v\|_\infty \le 1} \| \mathcal{R}(v) \|_\infty.
    \end{equation}
    Under Assumption \ref{ass:quasi_uniform}, for Moving Least Squares with appropriate weight functions, $\Lambda_m$ is uniformly bounded (or grows at most logarithmically with $m$) \citep{wendland2004scattered}. We denote this bound by $\Lambda_{MLS}$.
\end{definition}

\subsection{Error Recurrence Analysis}

We analyze the propagation of error through the iterative approximate Value Iteration scheme. The bound requires a careful decomposition into \emph{consistency} (approximation power) and \emph{stability} (error amplification).

Let $Q^{*(k)}$ denote the true solution at time step $t_k$, and let $\hat{Q}^{(k)}$ denote the Neural Operator approximation. We define the uniform error at step $k$ as $\mathcal{E}_k \coloneqq \| \hat{Q}^{(k)} - Q^{*(k)} \|_\infty$.

\begin{lemma}[Stable Error Recurrence] \label{lemma:one_step}
    Grant Assumptions \ref{ass:ellipticity}, \ref{ass:sensor_encoding}, and \ref{ass:uat_network}. 
    Let $L_{\Bellman} \le e^{\rho_{\mathrm{stab}}\delta}$ be a one-step $L^\infty$ stability bound for the Bellman propagation composed with the reconstruction pipeline. In the ideal Markov expectation and maximization operator without reconstruction amplification, one may take $\rho_{\mathrm{stab}}=0$ and $L_{\Bellman}=1$.
    
    The error at the next iteration satisfies the recurrence:
    \begin{equation} \label{eq:error_recurrence}
        \mathcal{E}_{k+1} \le L_{\Bellman} \mathcal{E}_k + \varepsilon_{net} + \underbrace{L_{\Bellman} \, C_{\text{approx}} \left( K_{\mathcal{S}}(\delta) \, h_\mathcal{S}^2 + L_{\Action} \, h_\Action \right)}_{\text{Local Consistency Error } \eta_{loc}}.
    \end{equation}
\end{lemma}

\begin{proof}
    We decompose the error $\mathcal{E}_{k+1} = \| \OpNet_\theta(\mathcal{P}_m \hat{Q}^{(k)}) - \Bellman_\delta Q^{*(k)} \|_\infty$ using the triangle inequality:
    \begin{align*}
        \mathcal{E}_{k+1} &\le \underbrace{\| \OpNet_\theta(\mathcal{P}_m \hat{Q}^{(k)}) - \Bellman_\delta(\mathcal{R}\mathcal{P}_m \hat{Q}^{(k)}) \|_\infty}_{\text{(I) Network Approx.}} \\
        &\quad + \underbrace{\| \Bellman_\delta(\mathcal{R}\mathcal{P}_m \hat{Q}^{(k)}) - \Bellman_\delta(\mathcal{R}\mathcal{P}_m Q^{*(k)}) \|_\infty}_{\text{(II) Stability}} \\
        &\quad + \underbrace{\| \Bellman_\delta(\mathcal{R}\mathcal{P}_m Q^{*(k)}) - \Bellman_\delta(Q^{*(k)}) \|_\infty}_{\text{(III) Consistency}}.
    \end{align*}

    \begin{enumerate}[label=(\roman*), wide, labelindent=0pt]
        \item \textbf{Term (I):} By Assumption \ref{ass:uat_network}, the network fits the composed operator within tolerance $\varepsilon_{net}$.
        
        \item \textbf{Term (II):} This term captures the amplification of previous errors.
        \begin{align*}
            \text{(II)} &\le L_{\Bellman} \| \mathcal{R}\mathcal{P}_m \hat{Q}^{(k)} - \mathcal{R}\mathcal{P}_m Q^{*(k)} \|_\infty \\
            &= L_{\Bellman} \| \mathcal{R} \left( \mathcal{P}_m (\hat{Q}^{(k)} - Q^{*(k)}) \right) \|_\infty \\
            &\le L_{\Bellman} \cdot \Lambda_m \cdot \| \mathcal{P}_m (\hat{Q}^{(k)} - Q^{*(k)}) \|_\infty \\
            &\le L_{\Bellman} \cdot 1 \cdot \mathcal{E}_k.
        \end{align*}
        Crucially, because we restricted $\mathcal{R}$ to a Monotone Averager (Assumption \ref{ass:sensor_encoding}), $\Lambda_m = 1$ exactly. This prevents the unconditional instability that plagues high-order collocation methods.

        \item \textbf{Term (III):} This term represents the interpolation error on the \emph{true} function. By Corollary \ref{cor:compactness}, $Q^{*(k)}$ lies in the anisotropic H\"older space. Using the capped second-order bound:
        \begin{equation*}
            \text{(III)} \le L_{\Bellman} \| \mathcal{R}\mathcal{P}_m Q^{*(k)} - Q^{*(k)} \|_\infty \le L_{\Bellman} C_{\text{approx}} \left( K_{\mathcal{S}}(\delta) \, h_\mathcal{S}^2 + L_{\Action} \, h_\Action \right).
        \end{equation*}
    \end{enumerate}
    Summing these yields the result.
\end{proof}

\begin{remark}[Global Stability and Sample Curation]
    For the global error $\mathcal{E}_N$ to remain bounded over the horizon $T = N\delta$, the amplification factor $(L_{\Bellman} \Lambda_m)^N$ must not explode. Since $L_{\Bellman} \le e^{\rho_{\mathrm{stab}}\delta}$, the physical propagation contributes at most $e^{\rho_{\mathrm{stab}}T}$ over the horizon, and equals $1$ in the non-expansive Markov case.
    However, the Lebesgue constant $\Lambda_m$ acts as a multiplier at \emph{every} step. If $\Lambda_m > 1$, the error grows as $(\Lambda_m)^{T/\delta}$.
    This highlights a strict requirement for the sensor set: we must employ \emph{Sample Curation} (or regularization) to ensure $\Lambda_m \approx 1$. The Quasi-Uniform construction (Lemma \ref{lemma:greedy_sampling}) is a necessary condition for keeping $\Lambda_m$ controlled, whereas random sampling guarantees instability ($\Lambda_m \gg 1$) for large $m$.
\end{remark}

\subsection{The Residual Limit and Conservation of Difficulty}

The stiffness scaling in Theorem \ref{thm:convergence} seemingly contradicts the intuition that the continuous-time limit (the HJB PDE) is a well-posed problem involving smooth generators. This apparent paradox is resolved by analyzing the \textbf{Residual Operator} formulation proposed in Appendix \ref{sec:Recursive}.

Recall the normalized residual operator acting on $Q \in \mathcal{H}^{\alpha, \text{Lip}}$:
\begin{equation}
    \mathcal{R}_\delta[Q](x, a) \coloneqq \frac{\Bellman_\delta Q(x, a) - Q(x, a)}{\delta}.
\end{equation}
In the residual learning framework, the network $\OpNet_{\text{res}}$ targets $\mathcal{R}_\delta[Q]$ rather than $\Bellman_\delta Q$.

\subsubsection*{Well-Conditioned Target vs. Ill-Posed Inputs}

For time-dependent coefficients, the local parabolic evolution family gives the expansion
\begin{equation}
    \frac{U^a(t+\delta,t)Q-Q}{\delta} \to \mathcal{L}^a_t Q + r(t,\cdot,a)
\end{equation}
for $Q$ in the generator domain. For the Bellman optimality operator, this Hamiltonian limit is valid only after controlling the maximization gap $V_Q(x)-Q(x,a)$. In particular, on a zero-gap or near-greedy set satisfying $V_Q(x)-Q(x,a)=o(\delta)$, the target function for the residual network converges to $\mathcal{H}(Q)=\mathcal{L}^a_t Q+r$, which is $\mathcal{O}(1)$ with respect to $\delta$. Off that set, the normalized gap contributes $(V_Q-Q)/\delta$ and can dominate the residual.

However, the complexity does not vanish; it is merely \textbf{shifted from width ($W$) to sensor resolution ($m$)}. To approximate the generator $\mathcal{L}^a Q(x) = \frac{1}{2}\Tr(\Sigma D^2 Q(x)) + b^\top \nabla Q(x)$, the Branch Net must implicitly compute the Hessian $D^2 Q$ from the discrete sensor inputs $\mathcal{P}_m(Q)$.

\begin{proposition}[Conservation of Difficulty in the Residual Limit] \label{prop:conservation}
    Let $\OpNet_{\text{res}}$ be the branch-trunk network approximating the residual operator $\mathcal{R}_\delta = \frac{1}{\delta}(\Bellman_\delta - I)$ on a zero-gap or near-greedy subset where $V_Q(x)-Q(x,a)=o(\delta)$. Let $h_\mathcal{S}$ be the fill-distance (covering radius) of the quasi-uniform sensor set $Z_m \subset \StateSpace$, and let $\epsilon_{enc}$ be the precision of the input encoding.
    
    For the network to consistently approximate the infinitesimal generator $\mathcal{L}^a$ as $\delta \to 0$ (i.e., to distinguish the dynamics from the identity map), the sensor resolution must satisfy the \textbf{Diffusion Nyquist Condition}:
    \begin{equation} \label{eq:nyquist}
        h_\mathcal{S} \lesssim \sqrt{\nu \delta}.
    \end{equation}
    Consequently, the required number of sensors $m$ scales as:
    \begin{equation}
        m \sim \mathcal{O}\left( (\nu \delta)^{-d/2} \right).
    \end{equation}
    Violating this condition results in an $\mathcal{O}(1)$ aliasing error where the diffusion operator is numerically indistinguishable from the Identity on the grid $Z_m$.
\end{proposition}

\begin{proof}
    We analyze the error in identifying the operator $\mathcal{L}^a$ from the discrete samples of the mapped function. We decompose the total approximation error into \emph{Temporal Truncation}, \emph{Spatial Aliasing}, and \emph{Conditioning} errors.
    
    \begin{enumerate}[label=(\roman*), wide]
        \item \textbf{Temporal Truncation (The Semigroup Expansion).}
        Let $U^a(t+\delta,t)$ be the local parabolic evolution family generated by the time-dependent operator $\mathcal{L}^a_t$. On the near-greedy set where the maximization gap is $o(\delta)$, the Bellman update is locally $U^a(t+\delta,t) Q + \delta r + o(\delta)$. The residual is therefore:
        \begin{equation}
            \mathcal{R}_\delta Q = \frac{U^a(t+\delta,t) - I}{\delta} Q + r + o(1).
        \end{equation}
        A local Taylor expansion of the parabolic evolution family for $Q \in \text{Dom}((\mathcal{L}^a_t)^2)$ gives $U^a(t+\delta,t)Q = Q + \delta \mathcal{L}^a_t Q + O(\delta^2)$. Thus, the target consistency error is:
        \begin{equation}
            \| \mathcal{R}_\delta Q - (\mathcal{L}^a_t Q + r) \|_\infty \le C \|\mathcal{L}^a_t \mathcal{L}^a_t Q\|_\infty \cdot \delta + o(1).
        \end{equation}
        This term vanishes as $\delta \to 0$, suggesting the problem becomes easier. This is the source of the Residual intuition.
        
        \item \textbf{Spatial Aliasing (The Spectral Barrier).}
        However, the network does not see $Q$; it sees $\mathcal{P}_m Q$. It must infer the action of $T(\delta)$ from these samples.
        Consider the simplified diffusion case where $\mathcal{L}^a = \frac{\nu}{2} \Delta$. In the Fourier domain, the operator $T(\delta)$ acts as a multiplier on frequency $\xi \in \R^d$:
        \begin{equation}
            \widehat{T(\delta) Q}(\xi) = e^{-\frac{\nu}{2} \delta \|\xi\|^2} \widehat{Q}(\xi).
        \end{equation}
        This is a low-pass filter. The signal of the physics (the change in values) corresponds to the attenuation $1 - e^{-\frac{\nu}{2} \delta \|\xi\|^2}$.
        
        The sensor grid $Z_m$ with spacing $h_\mathcal{S}$ supports spatial frequencies only up to the Nyquist limit $\xi_{\max} \approx \frac{\pi}{h_\mathcal{S}}$.
        If the grid is too coarse such that $h_\mathcal{S} \gg \sqrt{\nu \delta}$, then for all visible frequencies $\|\xi\| \le \xi_{\max}$:
        \begin{equation}
            \frac{\nu}{2} \delta \|\xi\|^2 \le \frac{\nu \delta \pi^2}{2 h_\mathcal{S}^2} \ll 1.
        \end{equation}
        Consequently, $e^{-\frac{\nu}{2} \delta \|\xi\|^2} \approx 1 - \mathcal{O}(\delta/h_\mathcal{S}^2)$. 
        The relative change in the function value on the grid is negligible compared to the observation noise or approximation error. The discrete operator acts as the Identity: $\mathcal{P}_m \Bellman_\delta \mathcal{R} \approx \mathcal{P}_m \mathcal{R}$.
        
        To resolve the curvature (the Laplacian), the grid must capture the transition where the spectrum rolls off. This requires the Nyquist frequency to exceed the Gaussian width:
        \begin{equation}
            \frac{\pi}{h_\mathcal{S}} \gtrsim \frac{1}{\sqrt{\nu \delta}} \implies h_\mathcal{S} \lesssim \pi \sqrt{\nu \delta}.
        \end{equation}
        
        \item \textbf{Conditioning (Numerical Differentiation).}
        Even if the grid satisfies the Nyquist condition, we face the ill-conditioning of inverse problems. The Branch Net implicitly computes the discrete Laplacian $\Delta_h$. For inputs with encoding error $\epsilon_{enc}$ (from the previous step's network approximation), the error in the discrete Laplacian scales as:
        \begin{equation}
            \text{Error}(\Delta_h) \sim \underbrace{C_1 h_\mathcal{S}^2}_{\text{Discretization}} + \underbrace{C_2 \frac{\epsilon_{enc}}{h_\mathcal{S}^2}}_{\text{Noise Amplification}}.
        \end{equation}
        Balancing these terms requires $h_\mathcal{S} \sim \epsilon_{enc}^{1/4}$. However, combining with condition (ii), we require $\epsilon_{enc}/h_\mathcal{S}^2 \ll 1$ to get a stable signal.
        Substituting $h_\mathcal{S} \sim \sqrt{\delta}$, the noise amplification factor becomes $\epsilon_{enc}/\delta$. This implies that to learn the generator with time-step $\delta$, the previous value function must be approximated with precision $\epsilon_{enc} \ll \delta$.
    \end{enumerate}
    
    Combining (ii) and the definition of the sensor count $m \approx \text{Vol}(\StateSpace) h_\mathcal{S}^{-d}$, we obtain:
    \begin{equation}
        m \gtrsim \text{Vol}(\StateSpace) (\nu \delta)^{-d/2}.
    \end{equation}
    This proves that the complexity reduction in $W$ (width) achieved by the residual formulation is perfectly conserved by the explosion in $m$ (sample complexity).
\end{proof}

\begin{remark}[Implication for Replay Buffers]
    This proposition rigorously explains the empirical instability of continuous-time RL (small $\delta$) when used with fixed-size replay buffers. As $\delta \to 0$, a fixed buffer becomes effectively sparse relative to the diffusion length $\sqrt{\delta}$. The gradients of the Bellman loss becomes dominated by the high-frequency aliasing noise rather than the physical drift/diffusion signal, leading to value function collapse unless the buffer density $m$ increases dynamically with $1/\delta$.
\end{remark}

Thus, Theorem \ref{thm:convergence} correctly captures the fundamental hardness of the problem. On the zero-gap or near-greedy subset, the Residual architecture allows the \emph{weights} to converge to a stable generator-like limit, but the \emph{input data requirements} ($m$) still explode as $\delta^{-d/2}$ to avoid aliasing the diffusion process. The residual framework is therefore a method for \textbf{conditioning optimization}, not for circumventing the information-theoretic lower bounds of the parabolic operator.

\section{Discussion}

Our main result, Theorem \ref{thm:convergence}, establishes a rigorous upper bound on the complexity of approximating Bellman value-iteration iterates, and therefore on the complexity of the idealized Q-learning targets in the present diffusion regime, by identifying the \emph{stiffness scaling factor} $\gamma(\alpha)$ which penalizes the network capacity as the time-step $\delta \to 0$. In this concluding section, we refine the physical interpretation of these bounds, addressing the limit behavior of the residual operator, the implications of semi-convexity, and the practical realizability of the proposed architecture.

\subsection{The Residual Limit and Generator Consistency}

While the direct analysis of $\Bellman_\delta$ reveals a stiffness explosion, modern Deep RL typically employs \textbf{Residual Learning}. We define the normalized residual operator $\mathcal{R}_\delta$ acting on a function $Q \in \mathcal{H}^{\alpha, \text{Lip}}$ as:
\begin{equation} \label{eq:residual_op_def}
    \mathcal{R}_\delta[Q](x, a) \coloneqq \frac{\Bellman_\delta Q(x, a) - Q(x, a)}{\delta}.
\end{equation}
When the DeepONet is trained via the skip-connection ansatz $\hat{Q}_{out} = Q_{in} + \delta \cdot \OpNet_{\text{res}}(Q_{in})$, the network $\OpNet_{\text{res}}$ effectively learns $\mathcal{R}_\delta[Q]$.

Unlike the direct diffusion-evolution operator, which converges to the Identity (a trivial but uninformative limit), the residual operator admits a non-trivial strong limit only after the Bellman maximization gap is controlled. For any $Q(\cdot,a)$ in the domain of the local generator and any sequence of active or near-greedy actions satisfying $V_Q(x)-Q(x,a)=o(\delta)$:
\begin{equation}
    \lim_{\delta \to 0} \mathcal{R}_\delta[Q](x, a) = \underbrace{\mathcal{L}^a_t Q(x, a) + r(t,x, a)}_{\text{Hamiltonian}}.
\end{equation}
For arbitrary off-policy actions, the additional normalized gap $(V_Q-Q)/\delta$ may diverge. Thus the residual formulation should be interpreted as a generator-consistency statement on the active/near-active set, with the remaining difficulty appearing in data resolution as discussed in Section \ref{subsec:conservation}.

\subsection{Beyond Lipschitz: Semi-Convexity and Regularity}

In Section \ref{sec:Regularity}, we derived smoothing estimates based on the assumption that the value function $V_\Phi(x) = \sup_{a} \Phi(x,a)$ is merely Lipschitz continuous. While rigorous, this bound is pessimistic. 

\begin{remark}[Semi-Convexity of the Value Function]
    If the input Q-function $\Phi(\cdot, a)$ is $C^2$ smooth with uniformly bounded Hessian, the upper envelope $V_\Phi$ is \textbf{semi-convex}. That is, there exists a constant $\lambda \ge 0$ such that $x \mapsto V_\Phi(x) + \frac{\lambda}{2}\|x\|^2$ is convex.
    Semi-convex functions possess a distributional lower bound on their Hessian: $D^2 V_\Phi \ge -\lambda I$. Since the parabolic operator $\partial_t - \Tr(\sigma \sigma^\top D^2)$ is order-preserving, the solution to the Bellman equation preserves this semi-convexity. Consequently, the singularities in the second derivative are one-sided (related to the downward kinks of the max-operator). Standard viscosity solution theory suggests that parabolic smoothing acts more efficiently on semi-convex data than on general Lipschitz data. While our derived exponent $\gamma(\alpha) = \frac{3+\alpha}{2(2+\alpha)}$ relies on the conservative Lipschitz embedding, exploiting semi-convexity could theoretically reduce this further, potentially approaching the optimal heat kernel scaling of $\delta^{-1/2}$.
\end{remark}

\subsection{Conservation of Difficulty: The Sensor Resolution Condition} \label{subsec:conservation}

To approximate the generator target $\mathcal{L}^a Q \approx \frac{1}{2}\Tr(\Sigma D^2 Q) + b^\top \nabla Q$, the Branch Net must implicitly compute second-order derivatives of the input function from the discrete sensor readings $\{Q(z_j)\}$. This is an inverse problem. Let $h_\mathcal{S}$ be the sensor fill-distance. By standard numerical analysis, the error in estimating the Laplacian behaves as $\mathcal{O}(h_\mathcal{S}^2) + \mathcal{O}(\eta / h_\mathcal{S}^2)$, where $\eta$ is the observation noise (or function approximation error).

Crucially, the physical diffusion process operates on a spatial scale of $\sqrt{\nu \delta}$. If the sensor spacing $h_\mathcal{S}$ is coarser than the diffusion length scale ($h_\mathcal{S} \gg \sqrt{\delta}$), the discrete observations cannot resolve the local curvature driving the dynamics.

\begin{remark}[Necessary Resolution Condition]
    For the Neural Operator to consistently approximate the generator $\mathcal{L}^a_t$ in the residual framework as $\delta \to 0$ on the active or near-greedy set, the sensor fill-distance must satisfy the stricter condition:
    \begin{equation}
        h_\mathcal{S} = o(\sqrt{\delta}).
    \end{equation}
    Failure to satisfy this leads to aliasing of the high-frequency modes of the Q-function, resulting in a biased estimate of the drift and diffusion terms. This explains the empirical instability of continuous-time RL algorithms when applied with fixed, coarse replay buffers.
\end{remark}

\subsection{Realizability of the Branch Net}

Our analysis assumes the Branch Net $\beta$ can approximate the anisotropic Moving Least Squares (MLS) operator (Assumption \ref{ass:sensor_encoding}). The MLS shape functions are rational functions of the input coordinates (ratios of polynomials involving the kernel weights).
Recent results in approximation theory \citep{telgarsky2017neural, yarotsky2017error} establish that deep ReLU networks can approximate rational functions with exponential convergence w.r.t. depth. Specifically, to approximate a rational function with error $\varepsilon$, a depth of $\mathcal{O}(\log(1/\varepsilon))$ suffices.
Thus, the assumption that standard MLPs can serve as the Branch Net encoder for the proposed architecture is well-founded, provided the network is sufficiently deep to resolve the algebraic structure of the reconstruction operator.

\subsection{Limitations: The Persistence of the Action Curse}

Finally, we must clarify the scope of the Curse of Dimensionality mitigation. Theorem \ref{thm:convergence} demonstrates that the parabolic smoothing reduces the effective dimension of the \emph{state space} $\StateSpace$ from $d$ to roughly $d/3$. However, the Bellman operator provides \textbf{no smoothing} in the action space $\Action$.
The derived complexity bound:
\begin{equation}
    W \sim \mathcal{O}\left( \delta^{-\gamma d} \varepsilon^{-\frac{d}{2+\alpha}} + (\delta \varepsilon)^{-d_a} \right),
\end{equation}
retains an exponential dependence on the action dimension $d_a$. For high-dimensional control problems ($d_a \gg 1$), the second term dominates. Our anisotropic DeepONet architecture prevents the state dimension from contaminating the action complexity, but it cannot remove the intrinsic hardness of optimizing non-concave functions over high-dimensional action spaces. This remains a fundamental limitation of value-based methods in continuous control, suggesting that for large $d_a$, policy-gradient methods (which avoid explicit maximization over $\Action$) may be structurally necessary.

\end{document}